\documentclass{dynafront2026} 

\usepackage{float} 
\usepackage{etoc}

\SetKwInput{KwRequire}{Require}
\SetKwComment{Comment}{$\triangleright$\ }{}
\LinesNumbered

\usepackage{booktabs}
\usepackage{multirow}
\usepackage{wrapfig}
\usepackage{placeins} 

\providecommand{\qedhere}{}

\title[Encoder–Decoder Manifold Alignment for Idempotent Generation]{Encoder–Decoder Manifold Alignment for Idempotent Generation}

\optauthor{%
\Name{Dareen Alharthi} \Email{dalharth@andrew.cmu.edu}\\
\Name{Ibrahim Aldarmaki} \Email{ialdarma@andrew.cmu.edu}\\
\Name{Abdul Waheed} \Email{abdulw@andrew.cmu.edu}\\
\addr Carnegie Mellon University, USA
\AND
\Name{Semyon Semenov} \Email{semen.semenov@mbzuai.ac.ae}\\
\addr Mohamed bin Zayed University of Artificial Intelligence, UAE
\AND
\Name{Bhiksha Raj} \Email{bhikshar@andrew.cmu.edu}\\
\addr Carnegie Mellon University, USA}

\begin{document}

\maketitle

\begin{abstract}

Recently, several learning paradigms have been introduced to enforce idempotency in generative models. The goal of this property is to ensure that repeated application of a model leaves samples unchanged once they lie on the target data manifold. In this work, we explore enforcing idempotency in encoder--decoder models. We argue that a key reason for the failure of generative models on some tasks, such as editing, is a geometric mismatch between the manifolds learned by the encoder and decoder. The encoder projects inputs onto one latent manifold, while the decoder implicitly learns to reconstruct data from a different manifold. This discrepancy prevents the model from learning truly idempotent mappings. To address this issue, we propose a new training framework that explicitly closes this gap by forcing the encoder and decoder to learn consistent representations of the same underlying data manifold. By aligning the geometry of these components, our method encourages stable projections. Empirically, we show that our approach achieves significantly lower idempotency error and consistently regenerates identical outputs under repeated application, compared to existing methods. We demonstrate the effectiveness of the proposed framework on both image generation and image editing tasks. Finally, we show that enforcing idempotency in this manner improves identity preservation and information stability, leading to more realistic and controllable generative editing models. \end{abstract}

\section{Introduction}\label{sec:intro}

Generative modeling has become a foundational tool across a wide range of data modalities, from images and speech to video \citep{rombach2022high, shen2023naturalspeech, ho2022video}. Modern high-performing generative models, including latent diffusion models \citep{rombach2022high, wu2025latentps}, flow matching \citep{lipman2022flow, hu2024latent}, and rectified flows \citep{liu2022flow, lupacscu2026optimal, kim2025reflex}, share a common architectural principle: rather than operating directly in high-dimensional data spaces such as pixel grids or audio waveforms, they learn a compact latent representation and perform generation within it. By compressing data through a learned encoder and reconstructing it through a learned decoder, these models delegate the complexity of the data distribution to a well-behaved latent space. The encoder and decoder are therefore core components of the generative pipeline, and their quality directly determines the fidelity and controllability of generated outputs \cite{rombach2022high,yao2025reconstruction}. 

\begin{figure}[H]
    \centering
    \includegraphics[width=0.7\linewidth]{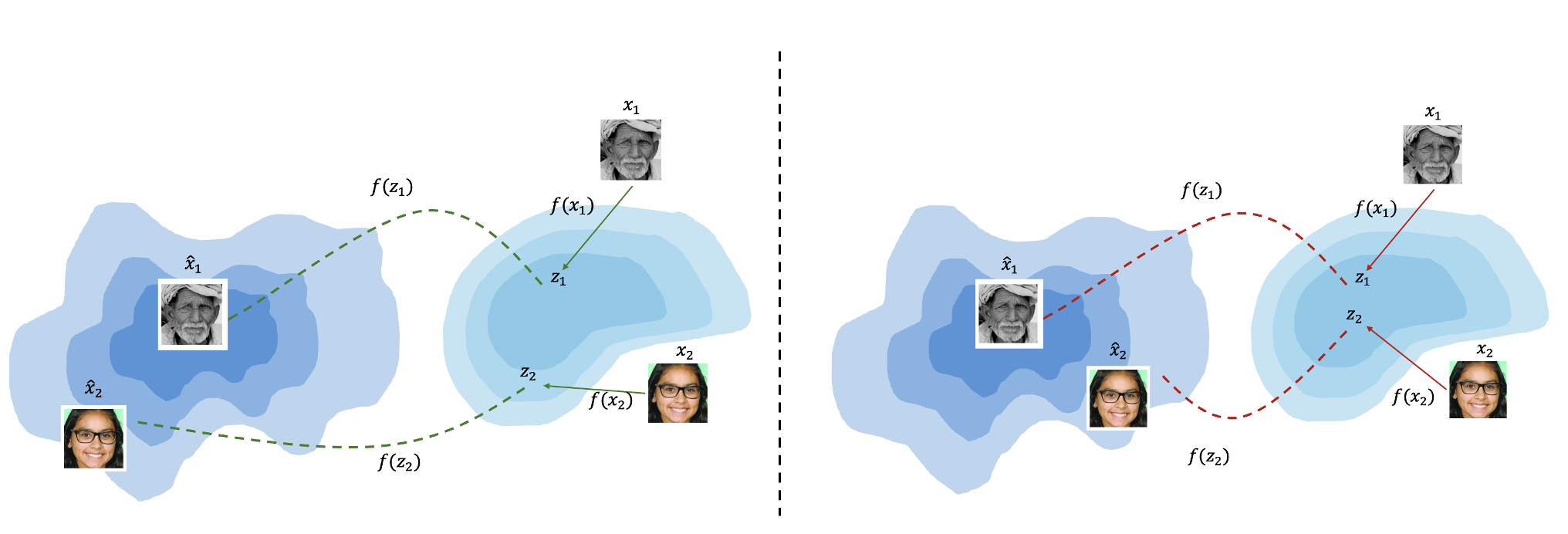}
    \vspace{-0.5em}
    \caption{
    Comparison between aligned and unaligned encoder--decoder manifolds.
    \textbf{Left:} Aligned manifold, where the encoder and decoder learn consistent representations of the same data manifold, leading to stable fixed points.
    \textbf{Right:} Unaligned manifold, where geometric mismatch between the encoder and decoder causes drift under repeated generation.
    }
    \label{fig:manifold_alignment}
\end{figure}

One of the most important downstream applications of generative models is \emph{editing}: taking an input such as an image, speech recording, or video and modifying specific attributes while leaving everything else intact \citep{abdal2021styleflow, anastassiou2024voiceshop, avrahami2025stable, yang2025videograin}. Editing is typically performed in latent space by encoding the input, modifying its latent representation, and decoding it back to data space. This process requires the encoder and decoder to share a consistent data manifold. Without alignment, latent edits lose localization, causing target modifications to corrupt unedited features like identity or fine details \citep{rout2024semantic, pan2025counterfactual}.

This raises a fundamental question: what property must a model satisfy to ensure stable, identity-preserving editing? We argue that a key requirement is \emph{idempotency}: for an encoder--decoder model $f = D_\theta \circ E_\phi$, applying the model twice should give the same result as applying it once, i.e., $f(f(x)) = f(x)$. This property exposes a failure that can occur even after a single pass. Specifically, if $E(D(E(x))) \neq E(x)$, the decoder changes the representation produced by the encoder, causing the reconstructed sample to shift in latent space when it is re-encoded. This shift can alter attributes that were not meant to be edited, leading to identity loss \cite{zhou2025multi,wang2026dit}.

As illustrated in Figure~\ref{fig:manifold_alignment}, this degradation comes from a geometric mismatch between the encoder and decoder manifolds. When misaligned, re-encoding $D(E(x))$ projects the reconstructed data onto a shifted latent representation, causing progressive drift across iterations. Conversely, aligning these manifolds ensures that the composition $D_\theta \circ E_\phi$ forms a stable projection whose fixed points lie directly on the learned manifold. Because projection operators are idempotent by definition \citep{taylor1980, deutsch2001best, federer2014geometric}, we argue that idempotency is a necessary  condition for an optimal information preserving autoencoder.

Prior works on idempotent models, such as Idempotent Generative Networks (IGN) \citep{shocher2023idempotent, zaman2025score}, enforce idempotency as a generative property operating directly in data space ($f: \mathcal{X} \rightarrow \mathcal{X}$) to map noisy or corrupted samples onto the data manifold in one step. Similarly, cycle consistency methods \citep{zhu2017unpaired} enforce approximate invertibility between two distinct mappings across domains ($G(F(x)) \approx x$). In contrast, our work formulates idempotency as a property  for encoder decoder models. Our framework aligns the internal representations of $E$ and $D$ within a single distribution, directly forcing the composition $D \circ E$ to behave as an idempotent operator.  In this paper, we show that idempotency is necessary for an encoder-decoder model to act as a projection onto the learned data manifold, introduce a simple method to align encoder and decoder representations, and demonstrate that this alignment improves identity preservation and editing success across different generative architectures.

\section{Method}\label{sec:method}

We consider an autoencoder composed of an encoder $E_\phi : \mathcal{X} \rightarrow \mathcal{Z}$ and decoder $D_\theta : \mathcal{Z} \rightarrow \mathcal{X}$. The standard mapping $f(x) = D_\theta(E_\phi(x))$ minimizes reconstruction error $\mathcal{L}_{\text{rec}} = \mathbb{E}_{x \sim p_{\mathcal{X}}} \big[ \| D_\theta(E_\phi(x)) - x \| \big]$. However, this objective alone permits a geometric mismatch between $E_\phi$ and $D_\theta$. This mismatch causes unedited image features to change after reencoding during latent editing (Figure~\ref{fig:model}, Appendix~\ref{sec:details}).

An optimal autoencoder projects input data onto a learned data manifold $\mathcal{M}$ via $P(x) = D_\theta(E_\phi(x))$. Projection operators are idempotent by definition \citep{taylor1980, deutsch2001best, federer2014geometric}, meaning applying $P$ twice must yield the exact same output ($P^2 = P$):
\begin{equation}
    D_\theta\!\left(E_\phi\!\left(D_\theta(E_\phi(x))\right)\right) = D_\theta(E_\phi(x)).
    \label{eq:output_idem}
\end{equation}
If $D(E(D(E(x)))) \neq D(E(x))$, the model fails to act as a projection because the reconstructed output shifts after a second pass through the network. A sufficient condition to guarantee equal outputs is matching the latent codes before and after decoding:
\begin{equation}
    E_\phi(D_\theta(E_\phi(x))) = E_\phi(x).
    \label{eq:latent_consistency}
\end{equation}
Assuming $D_\theta$ has a Lipschitz constant $L$, matching these latent codes directly limits output drift:
\begin{equation}
    \left\| D_\theta(E_\phi(D_\theta(E_\phi(x)))) - D_\theta(E_\phi(x)) \right\| \;\leq\; L \cdot \left\| E_\phi(D_\theta(E_\phi(x))) - E_\phi(x) \right\|.
\end{equation}

To enforce this match during training, let $\operatorname{sg}(\cdot)$ denote the stop gradient operator. For an input $x$, we set $z = E_\phi(x)$ and $\hat{z} = E_\phi\!\left(D_\theta(\operatorname{sg}(z))\right)$, defining the loss:
\begin{equation}
    \mathcal{L}_{\text{idem}} = \mathbb{E}_{x \sim p_{\mathcal{X}}} \left[ \left\| \hat{z} - \operatorname{sg}(z) \right\|_2^2 \right].
\end{equation}
The stop gradient holds the target code fixed, preventing the network from collapsing into a trivial identity function. Without $\operatorname{sg}(\cdot)$, the gradient contains $(1 - f'(f(x)))$, which becomes zero if $f'(z) = 1$. With $\operatorname{sg}(\cdot)$, the gradient becomes $-(\operatorname{sg}(f(x)) - f(f(x)))f'(f(x))f'(x)$, forcing $E_\phi$ and $D_\theta$ to align correctly.We combine this loss with any standard base objective $\mathcal{L}_{\text{base}}$ using weight $\lambda > 0$:
\begin{equation}
    \mathcal{L} = \mathcal{L}_{\text{base}} + \lambda\,\mathcal{L}_{\text{idem}}.
    \label{eq:total_loss}
\end{equation}
We provide complete formal proofs for these properties in Appendix~\ref{sec:proof}.

\begin{figure*}[t]
    \centering
    \includegraphics[width=0.72\textwidth]{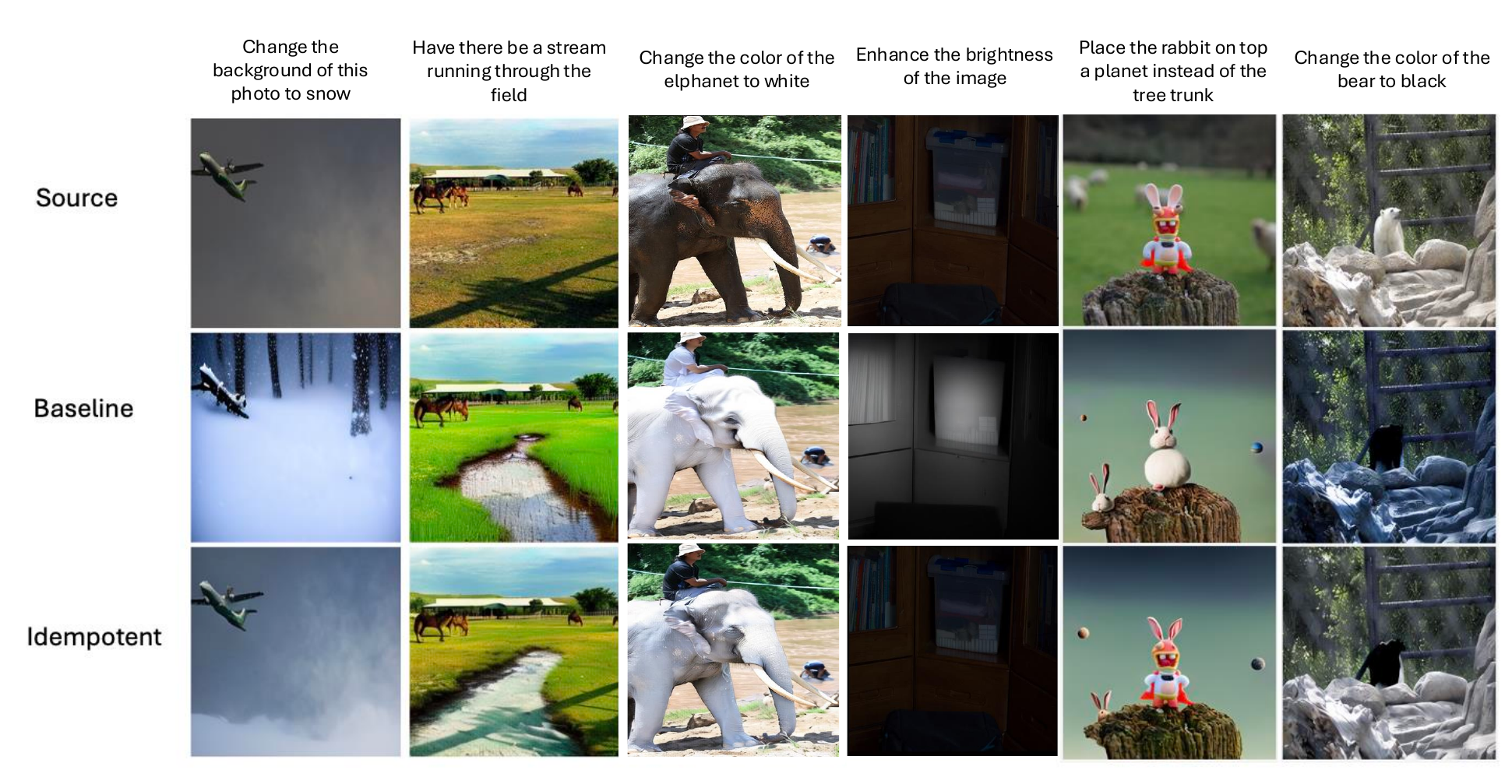}
    \vspace{-0.6em}
    \caption{Instruction-guided editing using a latent diffusion
model~\citep{brooks2023instructpix2pix} trained with a standard VAE (Baseline)
and our idempotent VAE (Idempotent). Our method better follows editing
instructions while preserving source image structure and details.}
    \label{fig:ldm_samples}
\end{figure*}

\begin{table*}[t]
\centering
\small
\setlength{\tabcolsep}{3.5pt}
\caption{Main results on image editing and reconstruction fidelity. Baseline and Idempotent InstructPix2Pix are evaluated across three benchmarks: I2EBench (500), MagicBrush (528), and BM-Bench (613). Bold values indicate better performance within each pair. MSE values are reported in $\times10^{-3}$.}
\label{tab:main_results}
\begin{tabular}{ll cccc cccccc}
\toprule
 & & \multicolumn{4}{c}{\textit{Editing (Source $\rightarrow$ Edited)}} & \multicolumn{6}{c}{\textit{VAE Reconstruction (Source vs.\ VAE(Source))}} \\
\cmidrule(lr){3-6} \cmidrule(lr){7-12}
Benchmark & Method & CLIP-I~$\uparrow$ & DINO~$\uparrow$ & SSIM~$\uparrow$ & LPIPS~$\downarrow$ & CLIP-I~$\uparrow$ & PSNR~$\uparrow$ & SSIM~$\uparrow$ & MSE~$\downarrow$ & LPIPS~$\downarrow$ \\
\midrule
\multirow{2}{*}{I2EBench}   & Baseline   & 0.856 & 0.760 & 0.620 & 0.357 & \textbf{0.951} & 27.63 & 0.743 & 3.32 & 0.142 \\
                             & Idempotent & \textbf{0.866} & \textbf{0.789} & \textbf{0.629} & \textbf{0.322} & 0.949 & \textbf{27.68} & \textbf{0.744} & \textbf{3.28} & \textbf{0.140} \\
\midrule
\multirow{2}{*}{MagicBrush} & Baseline   & 0.870 & 0.780 & 0.702 & 0.299 & \textbf{0.957} & 26.43 & 0.805 & 2.88 & 0.119 \\
                             & Idempotent & \textbf{0.872} & \textbf{0.782} & \textbf{0.705} & \textbf{0.297} & 0.954 & \textbf{26.51} & \textbf{0.807} & \textbf{2.84} & \textbf{0.117} \\
\midrule
\multirow{2}{*}{BM-Bench}   & Baseline   & 0.875 & 0.816 & 0.739 & 0.271 & \textbf{0.940} & 27.80 & 0.809 & 2.43 & 0.119 \\
                             & Idempotent & \textbf{0.903} & \textbf{0.864} & \textbf{0.772} & \textbf{0.228} & 0.938 & \textbf{27.87} & \textbf{0.811} & \textbf{2.39} & \textbf{0.118} \\
\bottomrule
\end{tabular}
\end{table*}

\section{Experimental Setup}
\label{sec:experiments}
To verify the manifold misalignment hypothesis we ran an experiemnt on six synthetic surfaces in $\mathbb{R}^3$, we train a baseline VAE and an
idempotent VAE with a two-dimensional latent, and ask whether the latents keep their structure by verifying if each point keeps the same neighbours, and whether the spacing between points
stays the same (Figure~\ref{fig:curve_latent_all}). The implementation details and code can be found in the appendix ~\ref{sec:details}.  For Image Editing we trained a latent diffusion model following the InstructPix2Pix
framework~\citep{brooks2023instructpix2pix}, substituting the standard VAE
with our idempotent VAE. We evaluate reconstruction quality and editing quality using  I2EBench \cite{ma2024i2ebench}, 
MagicBrush~\citep{zhang2023magicbrush} and BM-Bench \cite{chang2025bytemorph} benchmarks. Results for this experiment are shown in Figure~\ref{fig:ldm_samples} and Table~\ref{tab:main_results}. Also We evaluate idempotent generation on CelebA and MNIST.
A VQ-VAE-2~\citep{razavi2019generating} is trained on
CelebA~\citep{liu2018large} and tested on out-of-distribution samples from
LFW~\citep{zheng1708cross} to assess stability after repeated application
(Table~\ref{tab:reconstruction}). Following~\citep{shocher2023idempotent}, we also train a
DCGAN~\citep{radford2015unsupervised} on CelebA and compare our alignment
framework against the IGN baseline under input transformations including
Gaussian noise, grayscale, and sketching (see Figure~\ref{fig:celeba} in the Appendix). Additional experimental results, qualitative visualizations, and detailed ablation studies are provided in Appendix~\ref{sec:ablation}.

\section{Results}
\label{sec:results}

Figure~\ref{fig:curve_latent_all} shows a comparison of how the structure and neighbors for 1,000 samples change over 10 reconstruction iterations, as well as how the distances between points evolve. The figure demonstrates that without enforcing idempotency, the structure of the points changes and collapses after 10 iterations. Figure~\ref{fig:manifold_grid} and Table~\ref{tab:reconstruction} show that idempotency helps preserve more details from the original input while achieving lower error and variance after multiple reconstruction rounds. Figure~\ref{fig:ldm_samples} and Table~\ref{tab:main_results} show that an idempotent autoencoder helps the diffusion model edit target attributes while preserving other information in the image, such as the shadow in the second image. Overall, these experiments show that autoencoders are not idempotent operators by default, and enforcing this property leads to better results for generative models. Additional results and experimental details can be found in Appendix~\ref{sec:ablation}.

\begin{table*}[t]
\centering
\caption{Reconstruction stability after 30 encoder--decoder iterations.
We report mean $\pm$ standard deviation over 30 iterations for each metric.
Lower values are better for all metrics.}
\label{tab:reconstruction}
\small
\setlength{\tabcolsep}{6pt}
\renewcommand{\arraystretch}{0.85}
\begin{tabular}{llccc}
\toprule
Dataset & Method & FID $\downarrow$ & MSE $\downarrow$ & LPIPS $\downarrow$ \\
\midrule

\multirow{3}{*}{MNIST}
& VAE
& $1.70 \pm 1.06$
& $0.157 \pm 0.075$
& $0.191 \pm 0.041$ \\

& VAE + Ours
& $\mathbf{1.12 \pm 0.65}$
& $\mathbf{0.140 \pm 0.062}$
& $\mathbf{0.173 \pm 0.031}$ \\

\midrule

\multirow{2}{*}{CelebA}
& IGN~\cite{shocher2023idempotent}
& $54.68 \pm 16.71$
& $0.464 \pm 0.149$
& $0.273 \pm 0.036$ \\

& Ours
& $\mathbf{14.31 \pm 6.74}$
& $\mathbf{0.349 \pm 0.099}$
& $\mathbf{0.226 \pm 0.030}$ \\

\midrule

\multirow{3}{*}{LFW}
& VQ-VAE-2~\cite{razavi2019generating}
& $16.33 \pm 1.49$
& $0.0073 \pm 0.0026$
& $\mathbf{0.1046 \pm 0.0081}$ \\

& VQ-VAE-2 + Ours
& $\mathbf{17.27 \pm 0.14}$
& $\mathbf{0.0038 \pm 0.0001}$
& $0.1252 \pm 0.0018$ \\

\bottomrule
\end{tabular}
\end{table*}

\section{Conclusion}
\label{sec:conclusion}

We introduced a simple framework for enforcing idempotency in encoder--decoder models by aligning their learned manifolds. We show that idempotency is necessary for the encoder--decoder composition to act as an optimal projection, while latent-space consistency is sufficient to enforce it. Across image generation and editing tasks, our method reduces idempotency error and improves reconstruction fidelity, edit success, and identity preservation. These results highlight encoder--decoder alignment as a simple and effective objective for building more stable generative models.



\bibliography{sample}

\begin{thebibliography}{39}
\providecommand{\natexlab}[1]{#1}
\providecommand{\url}[1]{\texttt{#1}}
\expandafter\ifx\csname urlstyle\endcsname\relax
  \providecommand{\doi}[1]{doi: #1}\else
  \providecommand{\doi}{doi: \begingroup \urlstyle{rm}\Url}\fi

\bibitem[Abdal et~al.(2021)Abdal, Zhu, Mitra, and Wonka]{abdal2021styleflow}
Rameen Abdal, Peihao Zhu, Niloy~J Mitra, and Peter Wonka.
\newblock Styleflow: Attribute-conditioned exploration of stylegan-generated
  images using conditional continuous normalizing flows.
\newblock \emph{ACM Transactions on Graphics (ToG)}, 40\penalty0 (3):\penalty0
  1--21, 2021.

\bibitem[Anastassiou et~al.(2024)Anastassiou, Tang, Peng, Jia, Li, Tu, Wang,
  Wang, and Ma]{anastassiou2024voiceshop}
Philip Anastassiou, Zhenyu Tang, Kainan Peng, Dongya Jia, Jiaxin Li, Ming Tu,
  Yuping Wang, Yuxuan Wang, and Mingbo Ma.
\newblock Voiceshop: A unified speech-to-speech framework for
  identity-preserving zero-shot voice editing.
\newblock \emph{arXiv preprint arXiv:2404.06674}, 2024.

\bibitem[Avrahami et~al.(2025)Avrahami, Patashnik, Fried, Nemchinov, Aberman,
  Lischinski, and Cohen-Or]{avrahami2025stable}
Omri Avrahami, Or~Patashnik, Ohad Fried, Egor Nemchinov, Kfir Aberman, Dani
  Lischinski, and Daniel Cohen-Or.
\newblock Stable flow: Vital layers for training-free image editing.
\newblock In \emph{Proceedings of the Computer Vision and Pattern Recognition
  Conference}, pages 7877--7888, 2025.

\bibitem[Brooks et~al.(2023)Brooks, Holynski, and
  Efros]{brooks2023instructpix2pix}
Tim Brooks, Aleksander Holynski, and Alexei~A Efros.
\newblock Instructpix2pix: Learning to follow image editing instructions.
\newblock In \emph{Proceedings of the IEEE/CVF conference on computer vision
  and pattern recognition}, pages 18392--18402, 2023.

\bibitem[Caron et~al.(2021)Caron, Touvron, Misra, J{\'e}gou, Mairal,
  Bojanowski, and Joulin]{caron2021emerging}
Mathilde Caron, Hugo Touvron, Ishan Misra, Herv{\'e} J{\'e}gou, Julien Mairal,
  Piotr Bojanowski, and Armand Joulin.
\newblock Emerging properties in self-supervised vision transformers.
\newblock In \emph{Proceedings of the IEEE/CVF international conference on
  computer vision}, pages 9650--9660, 2021.

\bibitem[Chang et~al.(2025)Chang, Cao, Shi, Liu, Cai, Zhou, Huang, Wetzstein,
  Soleymani, and Wang]{chang2025bytemorph}
Di~Chang, Mingdeng Cao, Yichun Shi, Bo~Liu, Shengqu Cai, Shijie Zhou, Weilin
  Huang, Gordon Wetzstein, Mohammad Soleymani, and Peng Wang.
\newblock Bytemorph: Benchmarking instruction-guided image editing with
  non-rigid motions.
\newblock \emph{arXiv preprint arXiv:2506.03107}, 2025.

\bibitem[Deutsch and Deutsch(2001)]{deutsch2001best}
Frank Deutsch and F~Deutsch.
\newblock \emph{Best approximation in inner product spaces}, volume~7.
\newblock Springer, 2001.

\bibitem[Federer(2014)]{federer2014geometric}
Herbert Federer.
\newblock \emph{Geometric measure theory}.
\newblock Springer, 2014.

\bibitem[Heusel et~al.(2017)Heusel, Ramsauer, Unterthiner, Nessler, and
  Hochreiter]{heusel2017gans}
Martin Heusel, Hubert Ramsauer, Thomas Unterthiner, Bernhard Nessler, and Sepp
  Hochreiter.
\newblock Gans trained by a two time-scale update rule converge to a local nash
  equilibrium.
\newblock \emph{Advances in neural information processing systems}, 30, 2017.

\bibitem[Ho et~al.(2022)Ho, Salimans, Gritsenko, Chan, Norouzi, and
  Fleet]{ho2022video}
Jonathan Ho, Tim Salimans, Alexey Gritsenko, William Chan, Mohammad Norouzi,
  and David~J Fleet.
\newblock Video diffusion models.
\newblock \emph{Advances in neural information processing systems},
  35:\penalty0 8633--8646, 2022.

\bibitem[Hu et~al.(2024)Hu, Zhang, Tang, Mettes, Zhao, and Snoek]{hu2024latent}
Vincent~Tao Hu, Wei Zhang, Meng Tang, Pascal Mettes, Deli Zhao, and Cees Snoek.
\newblock Latent space editing in transformer-based flow matching.
\newblock In \emph{Proceedings of the AAAI conference on artificial
  intelligence}, volume~38, pages 2247--2255, 2024.

\bibitem[Kim et~al.(2025)Kim, Park, Song, Kwak, and Rhee]{kim2025reflex}
Jimyeong Kim, Jungwon Park, Yeji Song, Nojun Kwak, and Wonjong Rhee.
\newblock Reflex: Text-guided editing of real images in rectified flow via
  mid-step feature extraction and attention adaptation.
\newblock In \emph{Proceedings of the IEEE/CVF International Conference on
  Computer Vision}, pages 15939--15948, 2025.

\bibitem[LeCun et~al.(1998)LeCun, Cortes, and Burges]{lecun1998mnist}
Yann LeCun, Corinna Cortes, and Christopher~J.C. Burges.
\newblock The mnist database of handwritten digits.
\newblock \emph{Neural Networks}, 10\penalty0 (6):\penalty0 1117--1119, 1998.
\newblock URL \url{http://yann.lecun.com/exdb/mnist/}.

\bibitem[Lipman et~al.(2022)Lipman, Chen, Ben-Hamu, Nickel, and
  Le]{lipman2022flow}
Yaron Lipman, Ricky~TQ Chen, Heli Ben-Hamu, Maximilian Nickel, and Matt Le.
\newblock Flow matching for generative modeling.
\newblock \emph{arXiv preprint arXiv:2210.02747}, 2022.

\bibitem[Liu et~al.(2022)Liu, Gong, and Liu]{liu2022flow}
Xingchao Liu, Chengyue Gong, and Qiang Liu.
\newblock Flow straight and fast: Learning to generate and transfer data with
  rectified flow.
\newblock \emph{arXiv preprint arXiv:2209.03003}, 2022.

\bibitem[Liu et~al.(2018)Liu, Luo, Wang, and Tang]{liu2018large}
Ziwei Liu, Ping Luo, Xiaogang Wang, and Xiaoou Tang.
\newblock Large-scale celebfaces attributes (celeba) dataset.
\newblock \emph{Retrieved August}, 15\penalty0 (2018):\penalty0 11, 2018.

\bibitem[Lupa{\c{s}}cu and Stupariu(2026)]{lupacscu2026optimal}
Marian Lupa{\c{s}}cu and Mihai~Sorin Stupariu.
\newblock Optimal transport for rectified flow image editing: Unifying
  inversion-based and direct methods.
\newblock In \emph{Proceedings of the IEEE/CVF Winter Conference on
  Applications of Computer Vision}, pages 6764--6774, 2026.

\bibitem[Ma et~al.(2024)Ma, Ji, Ye, Lin, Wang, Zheng, Zhou, Sun, and
  Ji]{ma2024i2ebench}
Yiwei Ma, Jiayi Ji, Ke~Ye, Weihuang Lin, Zhibin Wang, Yonghan Zheng, Qiang
  Zhou, Xiaoshuai Sun, and Rongrong Ji.
\newblock I2ebench: A comprehensive benchmark for instruction-based image
  editing.
\newblock \emph{Advances in Neural Information Processing Systems},
  37:\penalty0 41494--41516, 2024.

\bibitem[Matthey et~al.(2017)Matthey, Higgins, Hassabis, and
  Lerchner]{dsprites17}
Loic Matthey, Irina Higgins, Demis Hassabis, and Alexander Lerchner.
\newblock dsprites: Disentanglement testing sprites dataset.
\newblock https://github.com/deepmind/dsprites-dataset/, 2017.

\bibitem[Pan and Bareinboim(2025)]{pan2025counterfactual}
Yushu Pan and Elias Bareinboim.
\newblock Counterfactual image editing with disentangled causal latent space.
\newblock In \emph{The Thirty-ninth Annual Conference on Neural Information
  Processing Systems}, 2025.

\bibitem[Radford(2015)]{radford2015unsupervised}
Alec Radford.
\newblock Unsupervised representation learning with deep convolutional
  generative adversarial networks.
\newblock \emph{arXiv preprint arXiv:1511.06434}, 2015.

\bibitem[Radford et~al.(2021)Radford, Kim, Hallacy, Ramesh, Goh, Agarwal,
  Sastry, Askell, Mishkin, Clark, et~al.]{radford2021learning}
Alec Radford, Jong~Wook Kim, Chris Hallacy, Aditya Ramesh, Gabriel Goh,
  Sandhini Agarwal, Girish Sastry, Amanda Askell, Pamela Mishkin, Jack Clark,
  et~al.
\newblock Learning transferable visual models from natural language
  supervision.
\newblock In \emph{International conference on machine learning}, pages
  8748--8763. PmLR, 2021.

\bibitem[Ravuri and Vinyals(2019)]{ravuri2019classification}
Suman Ravuri and Oriol Vinyals.
\newblock Classification accuracy score for conditional generative models.
\newblock \emph{Advances in neural information processing systems}, 32, 2019.

\bibitem[Razavi et~al.(2019)Razavi, Van~den Oord, and
  Vinyals]{razavi2019generating}
Ali Razavi, Aaron Van~den Oord, and Oriol Vinyals.
\newblock Generating diverse high-fidelity images with vq-vae-2.
\newblock \emph{Advances in neural information processing systems}, 32, 2019.

\bibitem[Rombach et~al.(2022)Rombach, Blattmann, Lorenz, Esser, and
  Ommer]{rombach2022high}
Robin Rombach, Andreas Blattmann, Dominik Lorenz, Patrick Esser, and Bj{\"o}rn
  Ommer.
\newblock High-resolution image synthesis with latent diffusion models.
\newblock In \emph{Proceedings of the IEEE/CVF conference on computer vision
  and pattern recognition}, pages 10684--10695, 2022.

\bibitem[Rout et~al.(2024)Rout, Chen, Ruiz, Caramanis, Shakkottai, and
  Chu]{rout2024semantic}
Litu Rout, Yujia Chen, Nataniel Ruiz, Constantine Caramanis, Sanjay Shakkottai,
  and Wen-Sheng Chu.
\newblock Semantic image inversion and editing using rectified stochastic
  differential equations.
\newblock \emph{arXiv preprint arXiv:2410.10792}, 2024.

\bibitem[Shen et~al.(2023)Shen, Ju, Tan, Liu, Leng, He, Qin, Bian,
  et~al.]{shen2023naturalspeech}
Kai Shen, Zeqian Ju, Xu~Tan, Eric Liu, Yichong Leng, Lei He, Tao Qin, Jiang
  Bian, et~al.
\newblock Naturalspeech 2: Latent diffusion models are natural and zero-shot
  speech and singing synthesizers.
\newblock In \emph{The Twelfth International Conference on Learning
  Representations}, 2023.

\bibitem[Shocher et~al.(2023)Shocher, Dravid, Gandelsman, Mosseri, Rubinstein,
  and Efros]{shocher2023idempotent}
Assaf Shocher, Amil Dravid, Yossi Gandelsman, Inbar Mosseri, Michael
  Rubinstein, and Alexei~A Efros.
\newblock Idempotent generative network.
\newblock \emph{arXiv preprint arXiv:2311.01462}, 2023.

\bibitem[Taylor and Lay(1980)]{taylor1980}
Angus~E. Taylor and David~C. Lay.
\newblock \emph{Introduction to Functional Analysis}.
\newblock John Wiley \& Sons, New York, 2nd edition, 1980.
\newblock ISBN 0471846465.

\bibitem[Wang et~al.(2026)Wang, Zhou, Wang, Xu, Qiu, He, and Li]{wang2026dit}
Xiaoce Wang, Sifan Zhou, Kaifei Wang, Leli Xu, Xuerui Qiu, Tao He, and Ming Li.
\newblock Why do dit editors drift? plug-and-play low frequency alignment in
  vae latent space.
\newblock \emph{arXiv preprint arXiv:2605.08250}, 2026.

\bibitem[Wu et~al.(2025)Wu, Murata, Takahashi, Wu, and
  Tsuruoka]{wu2025latentps}
Zilong Wu, Hideki Murata, Nayu Takahashi, Qiyu Wu, and Yoshimasa Tsuruoka.
\newblock Latentps: Image editing using latent representations in diffusion
  models.
\newblock In \emph{Proceedings of the Winter Conference on Applications of
  Computer Vision}, pages 167--176, 2025.

\bibitem[Yang et~al.(2025)Yang, Zhu, Fan, and Yang]{yang2025videograin}
Xiangpeng Yang, Linchao Zhu, Hehe Fan, and Yi~Yang.
\newblock Videograin: Modulating space-time attention for multi-grained video
  editing.
\newblock In \emph{The Thirteenth International Conference on Learning
  Representations}, 2025.

\bibitem[Yao et~al.(2025)Yao, Yang, and Wang]{yao2025reconstruction}
Jingfeng Yao, Bin Yang, and Xinggang Wang.
\newblock Reconstruction vs. generation: Taming optimization dilemma in latent
  diffusion models.
\newblock In \emph{2025 IEEE/CVF Conference on Computer Vision and Pattern
  Recognition (CVPR)}, pages 15703--15712. IEEE, 2025.

\bibitem[Zaman et~al.(2025)Zaman, Liu, and Chiu]{zaman2025score}
Shehtab Zaman, Chengyan Liu, and Kenneth Chiu.
\newblock Score-based idempotent distillation of diffusion models.
\newblock \emph{arXiv preprint arXiv:2509.21470}, 2025.

\bibitem[Zhang et~al.(2023)Zhang, Mo, Chen, Sun, and Su]{zhang2023magicbrush}
Kai Zhang, Lingbo Mo, Wenhu Chen, Huan Sun, and Yu~Su.
\newblock Magicbrush: A manually annotated dataset for instruction-guided image
  editing.
\newblock \emph{Advances in Neural Information Processing Systems},
  36:\penalty0 31428--31449, 2023.

\bibitem[Zhang et~al.(2018)Zhang, Isola, Efros, Shechtman, and
  Wang]{zhang2018unreasonable}
Richard Zhang, Phillip Isola, Alexei~A Efros, Eli Shechtman, and Oliver Wang.
\newblock The unreasonable effectiveness of deep features as a perceptual
  metric.
\newblock In \emph{Proceedings of the IEEE conference on computer vision and
  pattern recognition}, pages 586--595, 2018.

\bibitem[Zheng et~al.(2017)Zheng, Deng, and Hu]{zheng1708cross}
Tianyue Zheng, Weihong Deng, and Jiani Hu.
\newblock Cross-age {LFW}: A database for studying cross-age face recognition
  in unconstrained environments.
\newblock \emph{arXiv preprint arXiv:1708.08197}, 2017.

\bibitem[Zhou et~al.(2025)Zhou, Deng, He, Dong, and Tang]{zhou2025multi}
Zijun Zhou, Yingying Deng, Xiangyu He, Weiming Dong, and Fan Tang.
\newblock Multi-turn consistent image editing.
\newblock In \emph{2025 IEEE/CVF International Conference on Computer Vision
  (ICCV)}, pages 15792--15801. IEEE, 2025.

\bibitem[Zhu et~al.(2017)Zhu, Park, Isola, and Efros]{zhu2017unpaired}
Jun-Yan Zhu, Taesung Park, Phillip Isola, and Alexei~A Efros.
\newblock Unpaired image-to-image translation using cycle-consistent
  adversarial networks.
\newblock In \emph{Proceedings of the IEEE international conference on computer
  vision}, pages 2223--2232, 2017.

\end{thebibliography}

\newpage
\clearpage
\appendix

\section*{Appendix Contents}
\noindent
\hyperref[sec:details]{Appendix~\ref{sec:details}.\;Implementation Details}\dotfill\pageref{sec:details}\\
\hspace*{1.5em}\hyperref[sec:synthetic]{\ref{sec:synthetic}.\;Synthetic-Data Experiments}\dotfill\pageref{sec:synthetic}\\
\hspace*{1.5em}\hyperref[sec:alignment]{\ref{sec:alignment}.\;Manifold-Alignment Analysis}\dotfill\pageref{sec:alignment}\\
\hspace*{1.5em}\hyperref[sec:vqvae]{\ref{sec:vqvae}.\;VQ-VAE-2}\dotfill\pageref{sec:vqvae}\\
\hspace*{1.5em}\hyperref[sec:dcgan]{\ref{sec:dcgan}.\;DCGAN: IGN vs.\ Ours}\dotfill\pageref{sec:dcgan}\\
\hspace*{1.5em}\hyperref[sec:ip2p]{\ref{sec:ip2p}.\;InstructPix2Pix}\dotfill\pageref{sec:ip2p}\\
\hspace*{1.5em}\hyperref[sec:eval_metrics]{\ref{sec:eval_metrics}.\;Evaluation Metrics}\dotfill\pageref{sec:eval_metrics}\\
\hyperref[sec:proof]{Appendix~\ref{sec:proof}.\;Proofs}\dotfill\pageref{sec:proof}\\
\hspace*{1.5em}\hyperref[sec:proof_idem_necessary]{Idempotency is Necessary for Projection Optimality}\dotfill\pageref{sec:proof_idem_necessary}\\
\hspace*{1.5em}\hyperref[sec:proof_latent_consistency]{Latent Consistency is Sufficient for Output Idempotency}\dotfill\pageref{sec:proof_latent_consistency}\\
\hspace*{1.5em}\hyperref[sec:proof_latent_error]{Latent Error Bounds Output-Space Drift}\dotfill\pageref{sec:proof_latent_error}\\
\hyperref[sec:ablation]{Appendix~\ref{sec:ablation}.\;Additional Results}\dotfill\pageref{sec:ablation}
\bigskip

\providecommand{\toprule}{\hline}
\providecommand{\midrule}{\hline}
\providecommand{\bottomrule}{\hline}

\section{Implementation Details}

\addcontentsline{toc}{section}{Appendix~\thesection.\;Implementation Details}

\label{sec:details}

Figure~\ref{fig:model} illustrates the objective. An input $x$ is encoded to a
latent code $z = E(x)$ and decoded to a reconstruction $\hat{x} = D(z)$; the
latent code is then passed through a second decode--encode cycle,
$\hat{z} = E(D(\mathrm{sg}[z]))$, and we penalise the discrepancy between $z$
and $\hat{z}$. The stop-gradient prevents this term from updating the encoder
through the second pass, so the constraint is enforced in latent space rather
than in input space as in IGN \cite{shocher2023idempotent}. 

We implement the proposed objective on four autoencoders: a small VAE trained on
synthetic data, a VQ-VAE trained on CelebA, a DCGAN model that allows direct
comparison with IGN \cite{shocher2023idempotent}, and the InstructPix2Pix VAE
\cite{brooks2023instructpix2pix}, which we fine-tune on the InstructPix2Pix
dataset. Our implementations build on the publicly available code for
IGN,\footnote{\url{https://github.com/Alokia/Idempotent-Generative-Network}}
InstructPix2Pix,\footnote{\url{https://github.com/timothybrooks/instruct-pix2pix}}
and VQ-VAE-2.\footnote{\url{https://github.com/rosinality/vq-vae-2-pytorch}}
Each setting is described in a separate subsection below.

\subsection{Synthetic-Data Experiments}
\label{sec:synthetic}

In this experiments we used three data sources,  \textbf{Colored MNIST}~\citep{lecun1998mnist} and
\textbf{dSprites}~\citep{dsprites17}, which have labelled generative factors and
let us measure attribute editing, and \textbf{six 2-D manifolds} embedded in
$\mathbb{R}^3$, which have an exact data manifold and let us isolate the
geometric effect of the loss and to test if repeatedly applying the model
$f = D\!\circ\!E$ keeps points on the manifold. Table~\ref{tab:syn_data}
summarises the data.

\begin{table}[h]
\centering
\caption{Synthetic datasets.}
\label{tab:syn_data}
\small
\begin{tabular}{llll}
\toprule
Dataset & Sample & Factors & Size \\
\midrule
Colored MNIST & $3\times28\times28$ & digit ($10$), color ($10$) & $60$k / $10$k \\
dSprites      & $1\times64\times64$ & shape, scale, orient., posX, posY & $90$/$10$ split \\
Six manifolds & $\mathbb{R}^3$ (2-D) & swiss roll, S-curve, torus, & $6{,}000$ pts \\
              &                      & sphere, helix tube, trefoil tube & \\
\bottomrule
\end{tabular}
\end{table}

\begin{figure}[t]
    \centering
    \includegraphics[width=\textwidth]{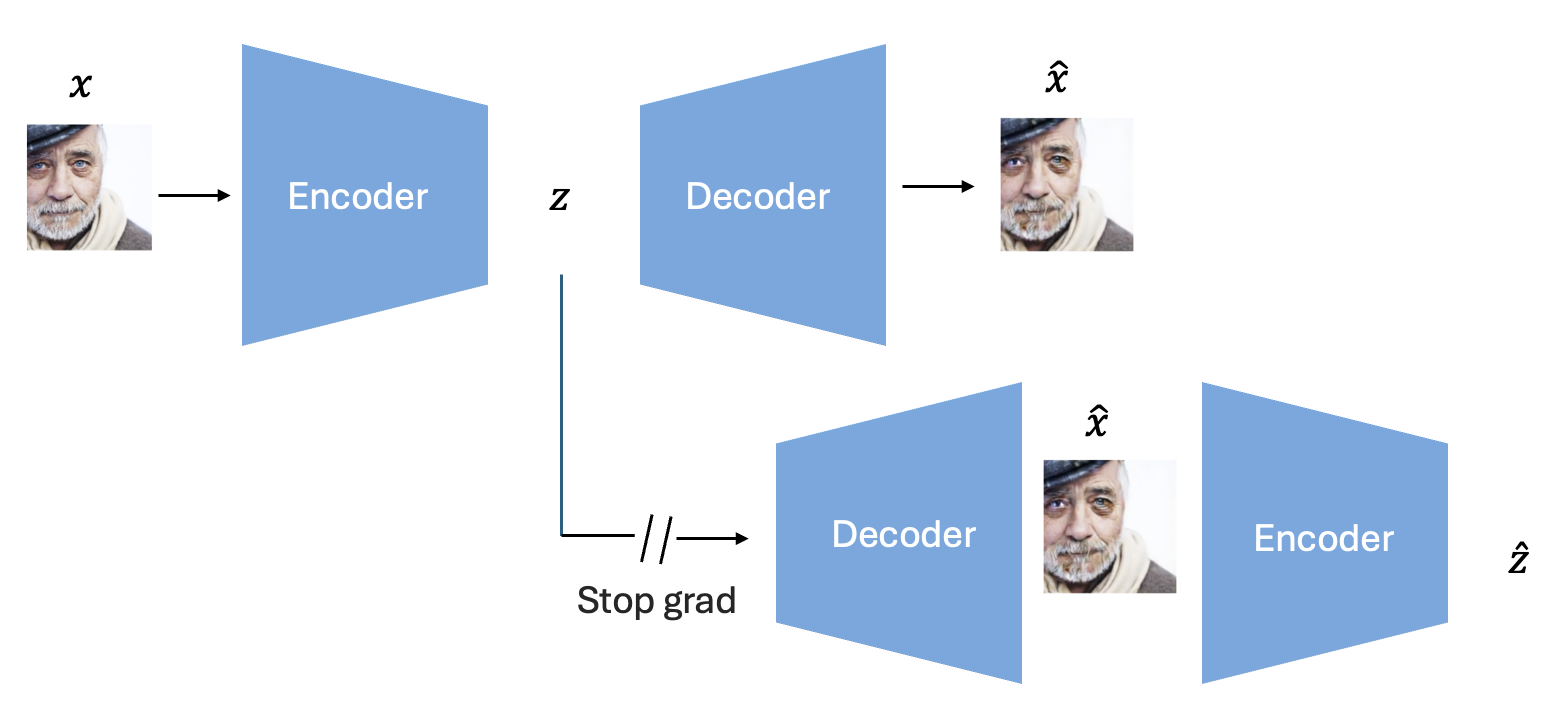}
    
\caption{
Idempotent training via encoder-decoder consistency.
An input sample $x$ is encoded into a latent representation $z$ and decoded to produce a reconstruction $\hat{x}$.
The reconstruction is then passed through a second decode-encode cycle with gradients stopped,
encouraging the encoder and decoder to learn consistent representations.
}
    \label{fig:model}
\end{figure}

\paragraph{Idempotency objective.}
All variants optimise the VAE ELBO augmented with a latent-consistency term
enforced \emph{in latent space} (Algorithm~\ref{alg:idem}). The term has two
parts: (i)~a data term that requires re-encoding a decoded posterior sample to
return its own code, and (ii)~a synthetic term that samples codes from the
posterior statistics at several scales and requires the encode--decode round
trip to fix them. Gradients through the second pass are stopped, so the
constraint shapes the map $E\!\circ\!D$.

\begin{algorithm2e}[H]
    \caption{Idempotency-regularised VAE training step}
    \label{alg:idem}
    \KwRequire{batch $x$, KL weight $\beta$, idempotency weight $\lambda$, term weights $\lambda_{\mathrm{data}}, \lambda_{\mathrm{syn}}$, scales $\mathcal{S} = \{0.5, 1, 2\}$}
    $\mu, \log \sigma^2 \gets E(x)$, \quad $z \sim \mathcal{N}(\mu, \operatorname{diag}(\sigma^2))$ \tcp*[r]{clean pass}
    $\mathcal{L}_{\mathrm{ELBO}} \gets \mathrm{recon}(D(z), x) + \beta \, \mathrm{KL}(q(z \mid x) \| p(z))$\;
    $\hat{\mu}, \widehat{\log \sigma^2} \gets E(D(\mathrm{sg}[z]))$ \tcp*[r]{re-encode reconstruction}
    $\mathcal{L}_{\mathrm{data}} \gets \lVert \hat{\mu} - \mathrm{sg}[\mu] \rVert_2^2 + \lVert \widehat{\log \sigma^2} - \mathrm{sg}[\log \sigma^2] \rVert_2^2$\;
    \For{$s \in \mathcal{S}$}{
      $z_s \sim \mathcal{N}(\mathrm{sg}[\mu], s^2 \operatorname{diag}(\sigma^2))$\;
      $\mathcal{L}_s \gets \lVert E(D(z_s)) - \mathrm{sg}[z_s] \rVert_2^2$\;
    }
    $\mathcal{L}_{\mathrm{syn}} \gets \frac{1}{\lvert \mathcal{S} \rvert} \sum_{s \in \mathcal{S}} \mathcal{L}_s$ \tcp*[r]{off-manifold latents}
    $\mathcal{L} \gets \mathcal{L}_{\mathrm{ELBO}} + \lambda (\lambda_{\mathrm{data}} \mathcal{L}_{\mathrm{data}} + \lambda_{\mathrm{syn}} \mathcal{L}_{\mathrm{syn}})$\;
    update $\theta, \phi$ by $\nabla \mathcal{L}$\;
\end{algorithm2e}

\paragraph{Attribute editing (Colored MNIST, dSprites).}
Editing follows a two-stage pipeline. We first train an unconditional VAE, then
freeze the encoder and train a conditional  flow $g_c$ on the encoder
means using learnable per-factor embeddings as the condition $c$. An edit maps a
latent to the base density under the source attributes and inverts under the
target attributes (Algorithm~\ref{alg:edit}). Configurations are in
Table~\ref{tab:syn_config}. The results for this experiment are shown in Figure~\ref{fig:cas_heatmap} and \ref{fig:editing_eval}. 

\begin{algorithm2e}[H]
    \caption{Attribute editing via the conditional flow}
    \label{alg:edit}
    \KwRequire{image $x$, source attributes $c_s$, target attributes $c_t$}
    $\mu \gets E(x)$;\quad $\epsilon \gets g_{c_s}(\mu)$ \tcp*[r]{flow forward under source}
    $\mu' \gets g^{-1}_{c_t}(\epsilon)$ \tcp*[r]{flow inverse under target}
    \KwRet{$D(\mu')$}\;
\end{algorithm2e}

\begin{table}[h]
\centering
\caption{Model and training configuration for the editing datasets.}
\label{tab:syn_config}
\small
\setlength{\tabcolsep}{6pt}
\begin{tabular}{lcc}
\toprule
 & \textbf{Colored MNIST} & \textbf{dSprites} \\
\midrule
VAE architecture / latent      & ConvVAE / $32$ & ResVAE64 / $64$ \\
VAE lr, batch, epochs          & $10^{-3}$, $2048$, $200$ & $10^{-3}$, $1024$, $20$ \\
Flow (FFJORD) blocks / hidden  & $4$ / $(256^3)$ & $4$ / $(256^3)$ \\
Flow lr, batch, epochs         & $5{\times}10^{-4}$, $8192$, $30$ & $5{\times}10^{-4}$, $4096$, $30$ \\
Condition embeddings           & digit,color $\!\times16$ & $5$ factors $\times16$ \\
Condition vocabularies         & $10,10$ & $[3,6,40,32,32]$ \\
\bottomrule
\end{tabular}
\end{table}

\paragraph{Manifold iteration (six manifolds).}
For each shape we sample $6{,}000$ points, standardise them, and train an MLP-VAE
(latent dim $=2$, hidden $128$, depth $3$; Adam, lr $10^{-3}$; $1200$ epochs;
$\beta=0.4$, $\lambda=10$). We then iterate the deterministic map
$f^{t}(x)=(D\!\circ\!E)^{t}(x)$ and measure how far the iterates leave the
manifold. Figure~\ref{fig:manifold_grid} shows the results for this experiment. 

\paragraph{Evaluation metrics.}
We report \textbf{CAS}~\citep{ravuri2019classification} (a classifier trained on
generated samples, tested on real data), \textbf{ESR} (a real-data classifier's
agreement that an edit hit its target attribute), \textbf{edit reversibility}
(MSE/LPIPS between an image and its edit-then-revert), and \textbf{idempotency}
(reconstruction MSE under $t\in\{10,30,50\}$ repeated encode--decode steps).

\subsection{Manifold-Alignment Analysis}
\label{sec:alignment}

The reconstruction map $P = D \circ E$ is meant to project onto the data
manifold, and a well-aligned encoder and decoder make $P$ idempotent, so that
iterating it stays put. We probe this alignment two ways: locally, through the
Jacobian of the latent round trip, and globally, by iterating $P$ and tracking
what happens to the latent arrangement.

\paragraph{Round-trip Jacobian.}
Let $h = E \circ D$ and $J = \partial h / \partial z \big|_{z_0}$ at
$z_0 = E(x)$. If the encoder inverts the decoder then $h$ is the identity and
$J = I$, so every latent perturbation returns unchanged. The singular values
$\sigma_{\max}(J) \ge \sigma_{\min}(J)$ give the gain in the best- and
worst-preserved directions; because they are invariant to rotation, a
relabelling of the latent coordinates is correctly not counted as a failure.
A singular value below one means part of a perturbation was lost on the round
trip, which happens exactly when the decoder moves the surface in a direction
the encoder does not respond to. We evaluate $J$ at $200$ points per surface,
as it varies across the manifold, and report the joint distribution of
$(\sigma_{\max}, \sigma_{\min})$ in Figure~\ref{fig:jacobian}. Isotropic
loss appears along the diagonal $\sigma_{\max} = \sigma_{\min}$; anisotropic
loss, where one direction survives and its orthogonal complement is crushed,
appears away from it.

\paragraph{Neighbourhood preservation under iteration.}
We apply $P$ to the data $t$ times and compare the resulting latents
$z_t = E(P^{t}(x))$ against $z_0$. Let $N^t_k(i)$ be the index set of the $k$
nearest neighbours of point $i$ at iteration $t$ ($k = 20$), and $\rho_t(i)$
its nearest-neighbour distance. Averaging over anchors,
\begin{align*}
\mathcal{J}(t) &= \Big\langle
  \tfrac{|N^0_k(i) \cap N^t_k(i)|}{|N^0_k(i) \cup N^t_k(i)|}
\Big\rangle_i,
&
\mathcal{R}(t) &= \Big\langle
  \tfrac{|N^0_k(i) \cap N^t_{2k}(i)|}{k}
\Big\rangle_i,
&
\mathcal{C}(t) &= \frac{\operatorname{med}_i \rho_t(i)}
                       {\operatorname{med}_i \rho_0(i)}.
\end{align*}
$\mathcal{J}$ penalises both neighbours that leave and new ones that arrive;
$\mathcal{R}$ widens the net to $2k$ after iteration and so penalises only
those that leave, making the two separable. $\mathcal{C}$ is one when the
latent cloud keeps its scale. Because nearest-neighbour sets are invariant to
uniform rescaling, measuring in the latent space separates reordering from
shrinkage rather than confounding them. Metrics are averaged over $10$ anchor
resamplings of $1{,}000$ points, for $t = 1, \dots, 10$.

For the idempotent model $\mathcal{J}$ and $\mathcal{R}$ stay high and
$\mathcal{C}$ stays near one throughout, while the baseline degrades
monotonically on both. The two analyses agree: the surfaces on which the
baseline's spectra depart furthest from identity are the same ones whose
neighbourhoods dissolve fastest under iteration, indicating that the drift
observed in the iteration grids
(Figure~\ref{fig:curve_latent_all}) originates in a local geometric
mismatch rather than in accumulated numerical error.

\begin{figure*}[t]
    \centering
    \includegraphics[width=\linewidth]{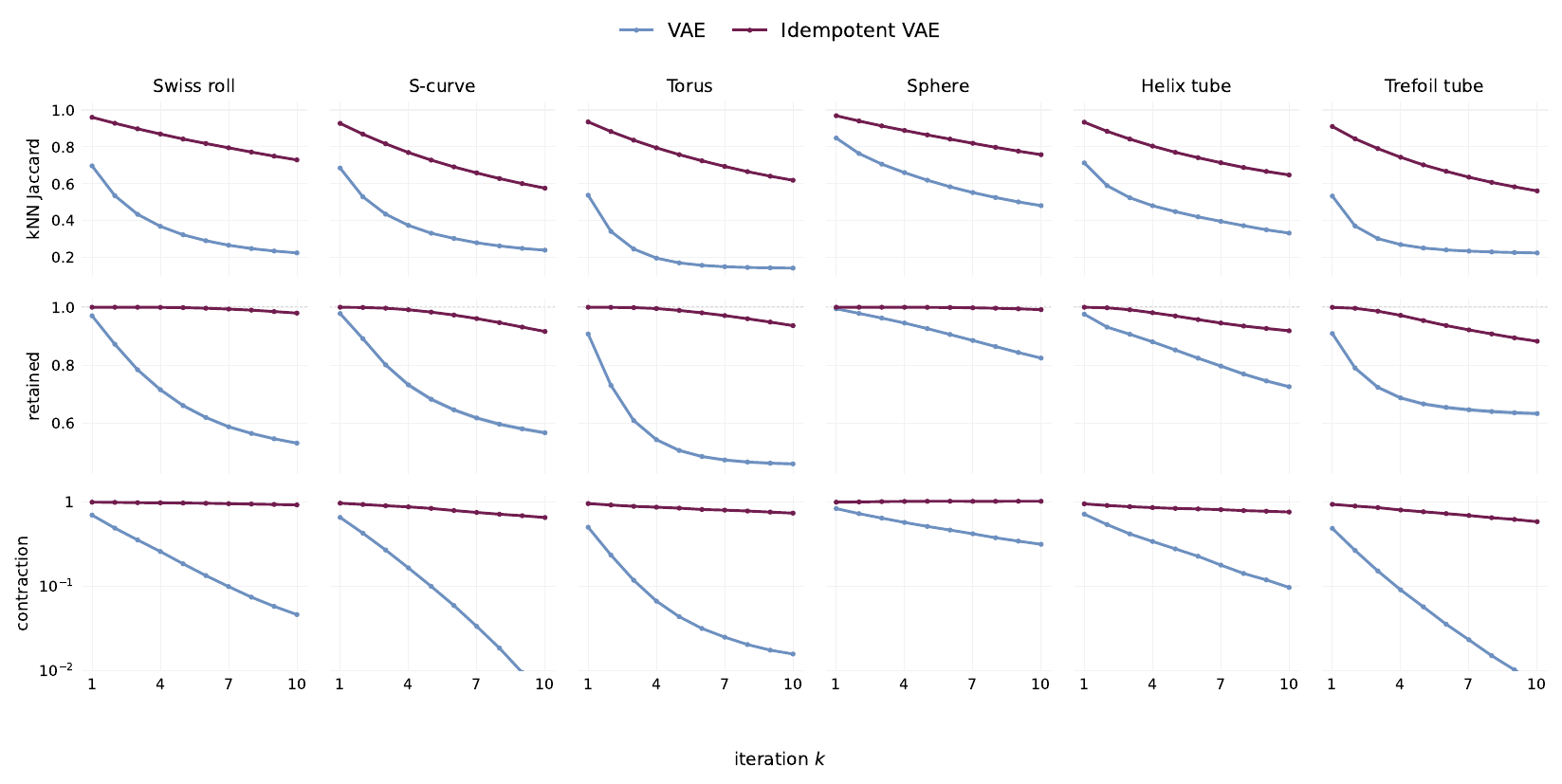}
    \caption{Neighbourhood preservation ($\mathcal{J}$, $\mathcal{R}$) and scale drift ($\mathcal{C}$) under repeated application of $f=D\circ E$ on six synthetic surfaces. The idempotent model preserves local structure and global scale substantially better than the baseline.}
    \label{fig:curve_latent_all}
\end{figure*}

\subsection{VQ-VAE-2}
\label{sec:vqvae}

We test the objective on a discrete-latent model, the two-level hierarchical
VQ-VAE-2~\citep{razavi2019generating}, trained on CelebA and evaluated on LFW at
$256\times256$ (images normalised to $[-1,1]$). The architecture uses $128$ base channels, two
residual blocks ($32$ channels), a $64$-dim embedding with a $256$-entry
codebook per level, and EMA codebook updates (decay $0.99$). The baseline
minimises reconstruction MSE plus a commitment loss (weight $0.25$); we train
with Adam (lr $3\times10^{-4}$).

Our variant adds a latent-consistency term on the \emph{quantised} codes. For
each level, decoding the (detached) codes and re-encoding must return the same
pre-quantisation features, and the same must hold for synthetic codes drawn from
the per-channel mean and standard deviation of the batch's codes:
\begin{equation*}
\mathcal{L}_{\text{idem}} =
\underbrace{\big\lVert E(D(\mathrm{sg}[q])) - \mathrm{sg}[z_e]\big\rVert^2}_{\text{data codes}}
+ 0.5\,\underbrace{\big\lVert E(D(q_{\text{syn}})) - q_{\text{syn}}\big\rVert^2}_{\text{synthetic codes}},
\qquad q_{\text{syn}} \sim \mathcal{N}(\bar q,\, \sigma_q^2),
\end{equation*}
summed over the top and bottom levels and added to the baseline loss (all other
settings unchanged). Both models are evaluated with the reconstruction metrics, the results are shown in   (Table~\ref{tab:reconstruction}).

\subsection{DCGAN: IGN vs.\ Ours}
\label{sec:dcgan}

To compare directly with IGN~\citep{shocher2023idempotent} we use its DCGAN
encoder--decoder on CelebA at $64\times64$, batch size $256$, AdamW (lr
$10^{-4}$, $\beta=(0.5,0.999)$). \textbf{IGN} applies its losses in image space:
reconstruction, idempotency $f(f(z))\!\approx\!f(z)$, and a smoothed tightness
term, evaluated on noise sampled from the batch statistics (weights
$\lambda_{\text{rec}}{=}10$, $\lambda_{\text{idem}}{=}10$,
$\lambda_{\text{tight}}{=}2.5$, $a{=}1.5$).

\textbf{Ours} keeps the same backbone but enforces idempotency \emph{in latent
space}: it re-encodes a decoded latent and matches the encoder's posterior
statistics,
\begin{equation*}
\mathcal{L}_{\text{idem}} = \big\lVert \mu\big(E(D(\mathrm{sg}[z]))\big) - \mathrm{sg}[\mu(E(x))] \big\rVert^2
+ \big\lVert \log\sigma^2(\cdot) - \mathrm{sg}[\log\sigma^2(E(x))] \big\rVert^2 ,
\end{equation*}
trained alongside an $\ell_1$ reconstruction loss, an adversarial loss with a
discriminator, and feature matching. Both methods are evaluated with the
reconstruction-stability protocol above (Table~\ref{tab:reconstruction}). Figure~\ref{fig:celeba} shows generation results for this experiment.

\subsection{InstructPix2Pix}
\label{sec:ip2p}

We also fine-tune the VAE of InstructPix2Pix~\citep{brooks2023instructpix2pix}
with the same latent-consistency objective; the implementation follows the
setups above, adding the latent-space idempotency term to the VAE's
reconstruction loss while leaving the backbone unchanged. Training uses the
InstructPix2Pix generated editing dataset (input image, instruction, edited
image), available from the authors' repository.\footnote{\url{https://github.com/timothybrooks/instruct-pix2pix}}

\subsection{Evaluation Metrics}
\label{sec:eval_metrics}

For VAE reconstruction we report PSNR and SSIM (pixel-level fidelity),
LPIPS~\citep{zhang2018unreasonable} (perceptual quality), and
CLIP-I~\citep{radford2021learning} (semantic preservation).
For instruction-based editing we additionally report
DINO~\citep{caron2021emerging} (fine-grained structural preservation) and
FID~\citep{heusel2017gans} (distributional realism of edits).
For synthetic benchmarks we use CAS~\citep{ravuri2019classification}
(class fidelity of generated samples) and ESR (attribute edit accuracy).
For idempotent generation we report FID, MSE, and LPIPS over repeated
reconstruction iterations.

\section{Proofs}
\addcontentsline{toc}{section}{Appendix~\thesection.\;Proofs}
\label{sec:proof}

We provide formal proofs for the theoretical claims made in Section \ref{sec:method}.

\subsection*{Idempotency is Necessary for Projection Optimality}
\phantomsection\label{sec:proof_idem_necessary}

\begin{definition}[Projection Operator]
A map $P : \mathcal{X} \to \mathcal{X}$ is a projection onto a set
$\mathcal{M} \subseteq \mathcal{X}$ if $P(x) \in \mathcal{M}$ for all $x$,
and $P(p) = p$ for all $p \in \mathcal{M}$.
\end{definition}

\begin{proposition}
\label{prop:necessary}
Let $f = D_\theta \circ E_\phi : \mathcal{X} \to \mathcal{X}$ be an
encoder--decoder composition that implements a projection onto a learned
manifold $\mathcal{M}$, i.e., $f(x) = \mathrm{proj}_{\mathcal{M}}(x)$.
Then $f$ must be idempotent: $f(f(x)) = f(x)$ for all $x \in \mathcal{X}$.
\end{proposition}

\begin{proof}
Since $f$ is a projection onto $\mathcal{M}$, we have $f(x) \in \mathcal{M}$
for all $x \in \mathcal{X}$.
By definition of a projection, any point already on $\mathcal{M}$ is fixed:
for all $p \in \mathcal{M}$, $f(p) = p$.
Setting $p = f(x) \in \mathcal{M}$ gives
\[
    f(f(x)) = f(x). \qedhere
\]
\end{proof}

\begin{remark}
Contrapositive: if $f(f(x)) \neq f(x)$ for some $x$, then $f$ does not
implement a valid projection onto any manifold, and the learned representation
is suboptimal in the projection sense.
\end{remark}
\subsection*{Latent Consistency is Sufficient for Output Idempotency}
\phantomsection\label{sec:proof_latent_consistency}

\begin{proposition}
\label{prop:sufficient}
Let $f = D_\theta \circ E_\phi$.
If the encoder satisfies the latent consistency condition
\begin{equation}
    E_\phi\!\left(D_\theta(E_\phi(x))\right) = E_\phi(x)
    \quad \forall\, x \in \mathcal{X},
    \label{eq:latent_consistency_app}
\end{equation}
then $f$ is idempotent:
$D_\theta(E_\phi(D_\theta(E_\phi(x)))) = D_\theta(E_\phi(x))$
for all $x \in \mathcal{X}$.
\end{proposition}

\begin{proof}
Let $z = E_\phi(x)$ and $\hat{x} = D_\theta(z)$.
By assumption~\eqref{eq:latent_consistency_app},
\[
    E_\phi(\hat{x}) = E_\phi(D_\theta(z)) = z.
\]
Applying $D_\theta$ to both sides:
\[
    D_\theta(E_\phi(\hat{x})) = D_\theta(z) = \hat{x}.
\]
Substituting back $\hat{x} = D_\theta(E_\phi(x))$:
\[
    D_\theta\!\left(E_\phi\!\left(D_\theta(E_\phi(x))\right)\right)
    = D_\theta(E_\phi(x)). \qedhere
\]
\end{proof}
\subsection*{Latent Error Bounds Output-Space Drift}
\phantomsection\label{sec:proof_latent_error}

\begin{proposition}
\label{prop:lipschitz}
Suppose $D_\theta$ is $L$-Lipschitz. Then the output-space idempotency
error is bounded by the latent-space idempotency error:
\[
    \left\|D_\theta(E_\phi(D_\theta(E_\phi(x)))) - D_\theta(E_\phi(x))\right\|
    \;\leq\;
    L\cdot\left\|E_\phi(D_\theta(E_\phi(x))) - E_\phi(x)\right\|.
\]
\end{proposition}

\begin{proof}
Let $z = E_\phi(x)$ and $\hat{z} = E_\phi(D_\theta(z))$.
Since $D_\theta$ is $L$-Lipschitz,
\[
    \left\|D_\theta(\hat{z}) - D_\theta(z)\right\|
    \leq L\,\|\hat{z} - z\|
    = L\,\left\|E_\phi(D_\theta(E_\phi(x))) - E_\phi(x)\right\|. \qedhere
\]
\end{proof}

\begin{corollary}
\label{cor:surrogate}
Under the assumptions of Proposition~\ref{prop:lipschitz}, any reduction
in $\mathcal{L}_{\mathrm{idem}} = \mathbb{E}_x\bigl[\|E_\phi(D_\theta(
\mathrm{sg}(z))) - \mathrm{sg}(z)\|^2\bigr]$ produces a proportional
reduction in output-space drift, with proportionality constant $L$.
This justifies minimizing the latent-space loss as a practical surrogate
for enforcing output-space idempotency.
\end{corollary}

\section{Additional Results}
\addcontentsline{toc}{section}{Appendix~\thesection.\;Additional Results}
\label{sec:ablation}

\paragraph{Idempotency loss weight.}
Figure~\ref{fig:mnist_sweep} ablates the idempotency loss weight $\lambda$ on
MNIST. Performance is stable across a wide range of values, with higher $\lambda$
generally improving FID while maintaining competitive LPIPS.

\begin{figure}[!http]
    \centering
    \includegraphics[width=\linewidth]{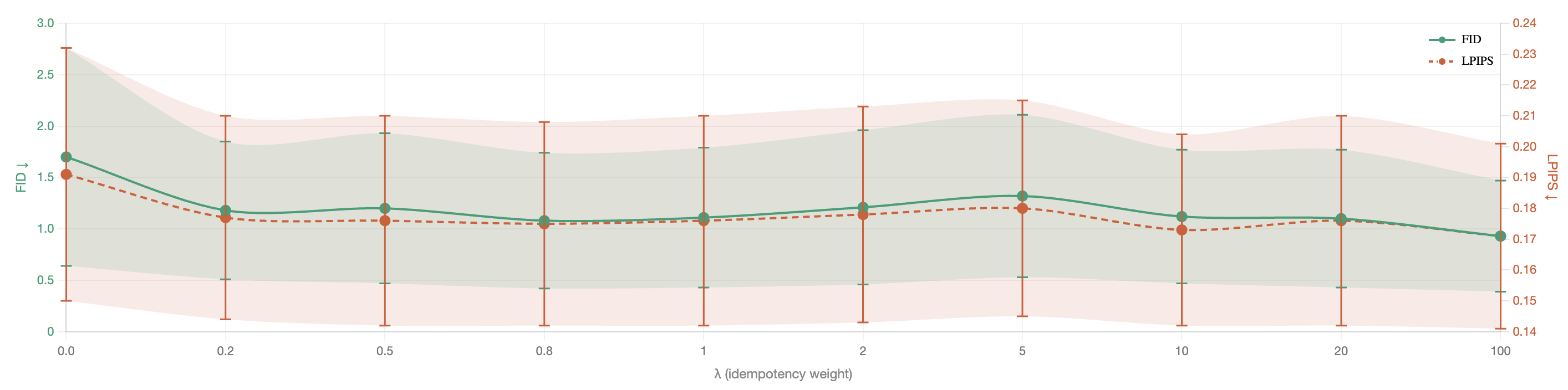}
    \caption{Ablation of the idempotency loss weight $\lambda$ on MNIST using
    VAE~+~Ours. FID and LPIPS are reported as mean~$\pm$~std over 30 iterations.
    Higher $\lambda$ generally improves FID while maintaining competitive LPIPS.}
    \label{fig:mnist_sweep}
\end{figure}

\paragraph{Latent interpolation.}
Figure~\ref{fig:mnist} compares latent interpolations across VAE, VAE~+~IGN,
and VAE~+~Ours on MNIST. Encoder--decoder alignment enables smooth, semantically
meaningful transitions: digits change gradually from 7 to 2 with clear intermediate
samples. Baseline and IGN interpolations contain noisy or ambiguous intermediates,
indicating latent space instability. This suggests alignment improves not only
idempotency but also the global geometry of the latent space.

\paragraph{Editing quality.}
Figure~\ref{fig:colored_mnist_com} shows qualitative editing results on Colored
MNIST. VAE~+~Ours preserves digit identity and color fidelity under repeated
application, whereas Baseline and IGN exhibit color drift and digit distortion.
Figure~\ref{fig:diff_edit} compares a latent diffusion model trained on idempotent
versus standard VAE representations. The idempotent representation preserves finer
details, including object positions, grass color, and sky color across edited scenes.

\paragraph{Robustness to input transformations.}
Figure~\ref{fig:celeba} evaluates DCGAN under input transformations (Gaussian
noise, grayscale, and sketching) following~\cite{shocher2023idempotent}.
Our method produces more consistent outputs and better preserves facial structure
under repeated application compared to IGN.

\begin{figure*}[t]
    \centering
    \includegraphics[width=\linewidth]{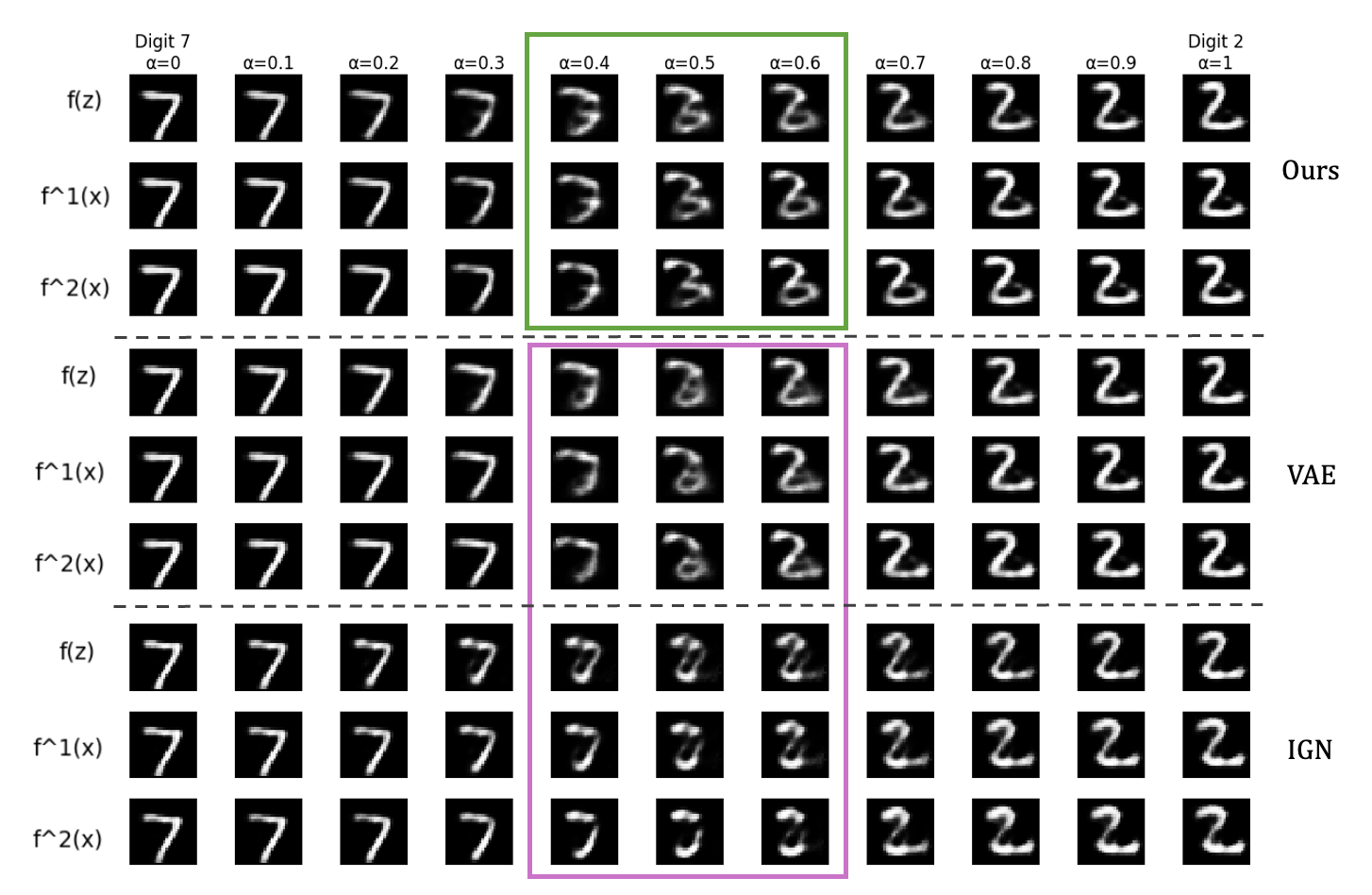}
    \caption{Latent interpolation under repeated generation on MNIST. Each column
    shows interpolation coefficient $\alpha$ between digits 7 and 2; each row shows
    repeated application of the same generator. Our method produces smooth, stable
    interpolations compared to VAE and IGN.}
    \label{fig:mnist}
\end{figure*}

\paragraph{Per-factor CAS.}
Figure~\ref{fig:cas_heatmap} breaks down CAS per generative factor on dSprites
and Colored MNIST. Our method achieves competitive CAS across most factors, with
the largest remaining gap on orientation in dSprites.
Figure~\ref{fig:qualitative} shows qualitative reconstruction comparisons across
VAE, VAE~+~IGN, and VAE~+~Ours, confirming that our method produces more
consistent and stable reconstructions.
\begin{figure*}[t]
    \centering
    \includegraphics[width=\linewidth]{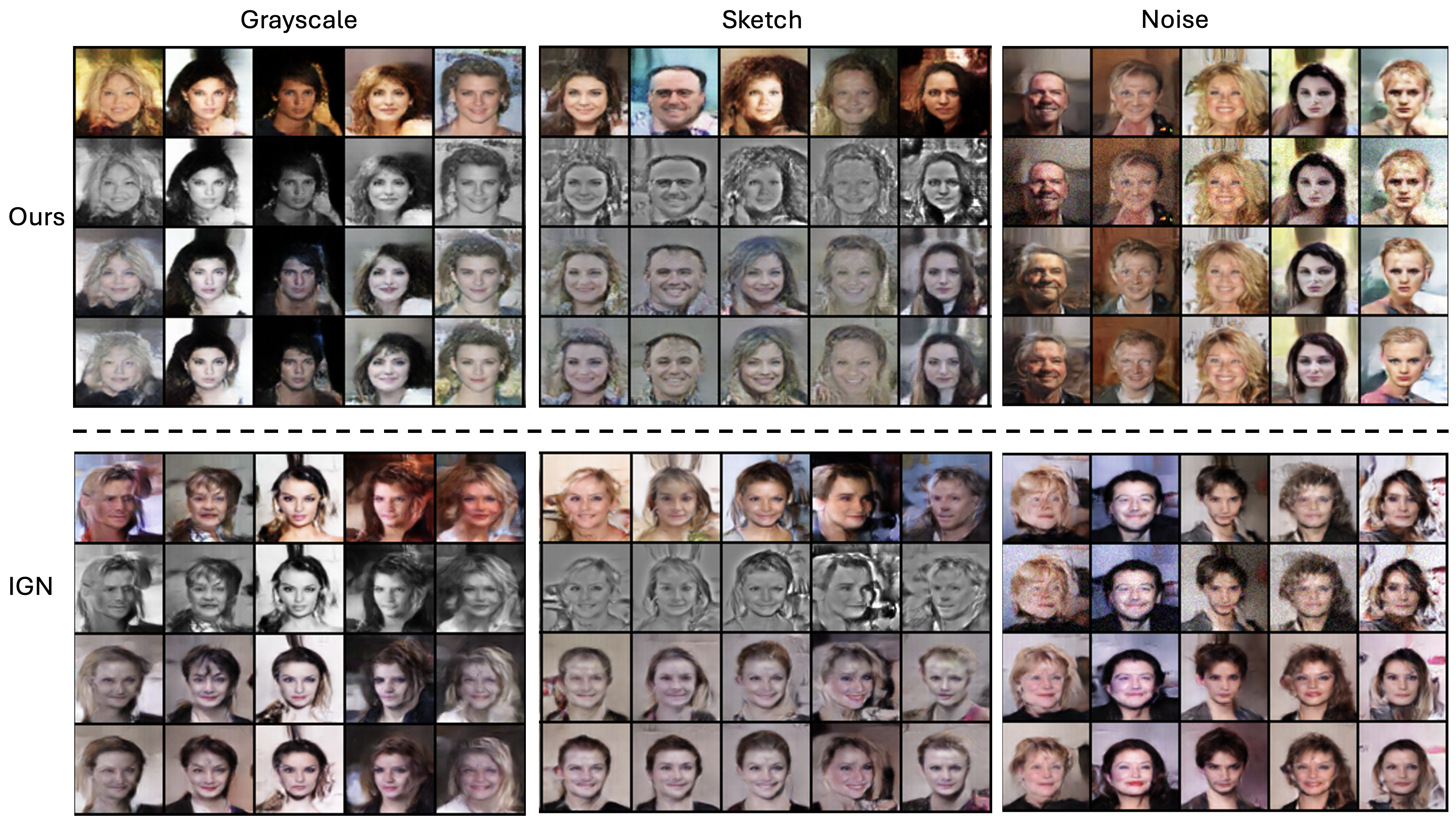}
    \caption{Idempotency under input transformations on CelebA using DCGAN.
    First row: generated images. Second row: transformed inputs (grayscale, sketch,
    noise) following~\cite{shocher2023idempotent}. Third and fourth rows: first and
    second re-applications of the model. Our method produces more consistent outputs
    and preserves facial structure better than IGN.}
    \label{fig:celeba}
\end{figure*}

\begin{figure*}[t]
    \centering
    \includegraphics[width=\linewidth]{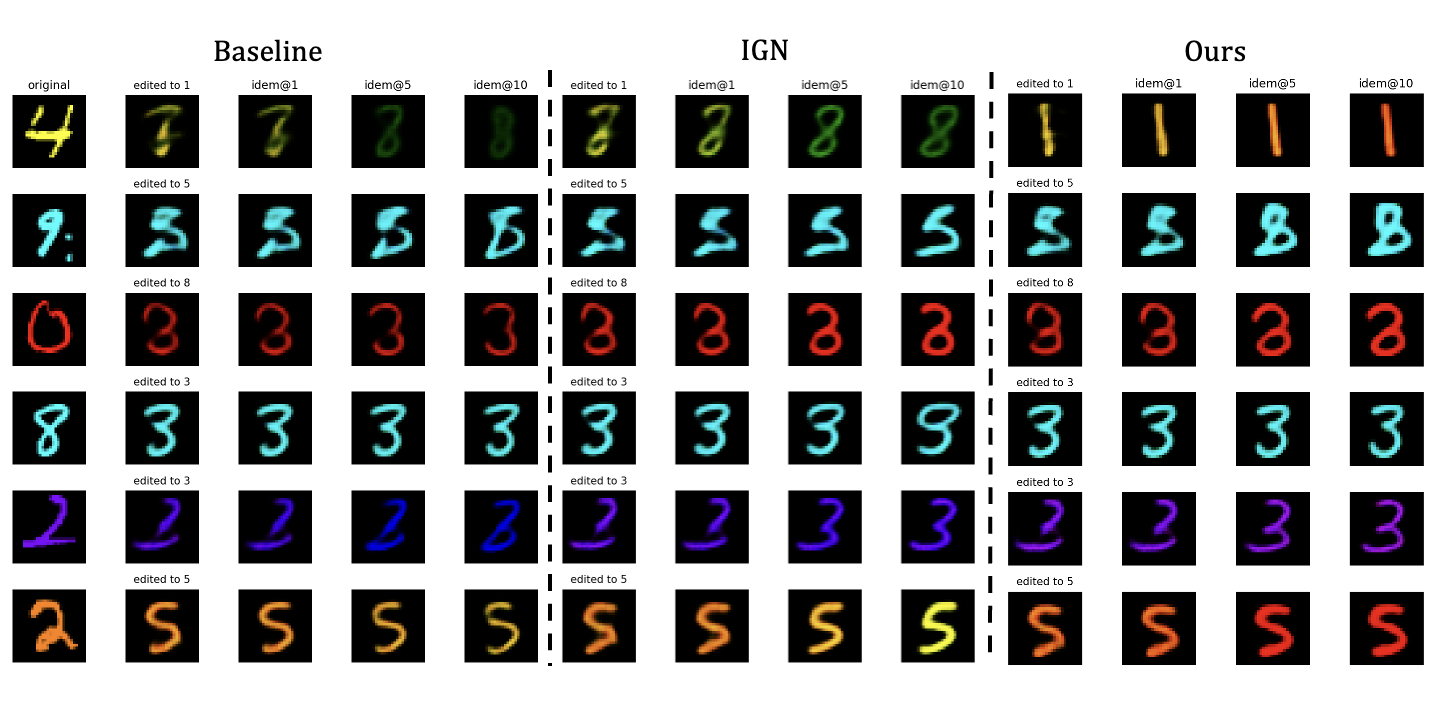}
    \caption{Qualitative editing results on Colored MNIST. Each row shows an
    original image, the edited result, and repeated applications at $f^{1}(x)$,
    $f^{5}(x)$, and $f^{10}(x)$. Baseline and IGN exhibit color drift and digit
    distortion, whereas our method preserves both digit identity and color
    fidelity across iterations.}
    \label{fig:colored_mnist_com}
\end{figure*}

\paragraph{Manifold iteration grids.}
Figure~\ref{fig:manifold_grid} visualises what repeated reconstruction does to
six synthetic surfaces in $\mathbb{R}^3$. After a single round both the
baseline and our model reproduce each surface faithfully, so a one-shot
comparison would find little to separate them. The difference appears only
under repetition: by the tenth round the baseline has collapsed each surface
onto a thin subset, losing both its extent and its internal ordering, while
our model returns a shape close to the original. Because colour encodes
position on the source manifold and is carried through the iterations, the
mixing visible in the baseline rows shows that points have moved relative to
one another rather than the cloud merely shrinking uniformly.

\paragraph{Round-trip Jacobian.}
Figure~\ref{fig:jacobian} locates the cause of that collapse in the local
geometry. If the encoder inverts the decoder, the Jacobian of the latent round
trip $E \circ D$ is the identity and both its singular values equal one, so
every perturbation returns unchanged. The baseline instead sits away from
$(1,1)$ with $\sigma_{\max}$ near one and $\sigma_{\min}$ up to two orders of
magnitude smaller, meaning one direction survives the round trip while its
orthogonal complement is destroyed: the decoder moves the surface in a
direction the encoder does not respond to. The loss is therefore anisotropic
rather than a uniform shrinkage, which is what folding a non-trivially
embedded surface into a two-dimensional latent produces. This measurement
requires no iteration, and the surfaces on which the spectra depart furthest
from identity are those whose neighbourhoods dissolve fastest in
Figure~\ref{fig:curve_latent_all}, indicating that the drift originates in a
local geometric mismatch rather than in accumulated numerical error.

\begin{figure}[t]
    \centering
    \includegraphics[width=\linewidth]{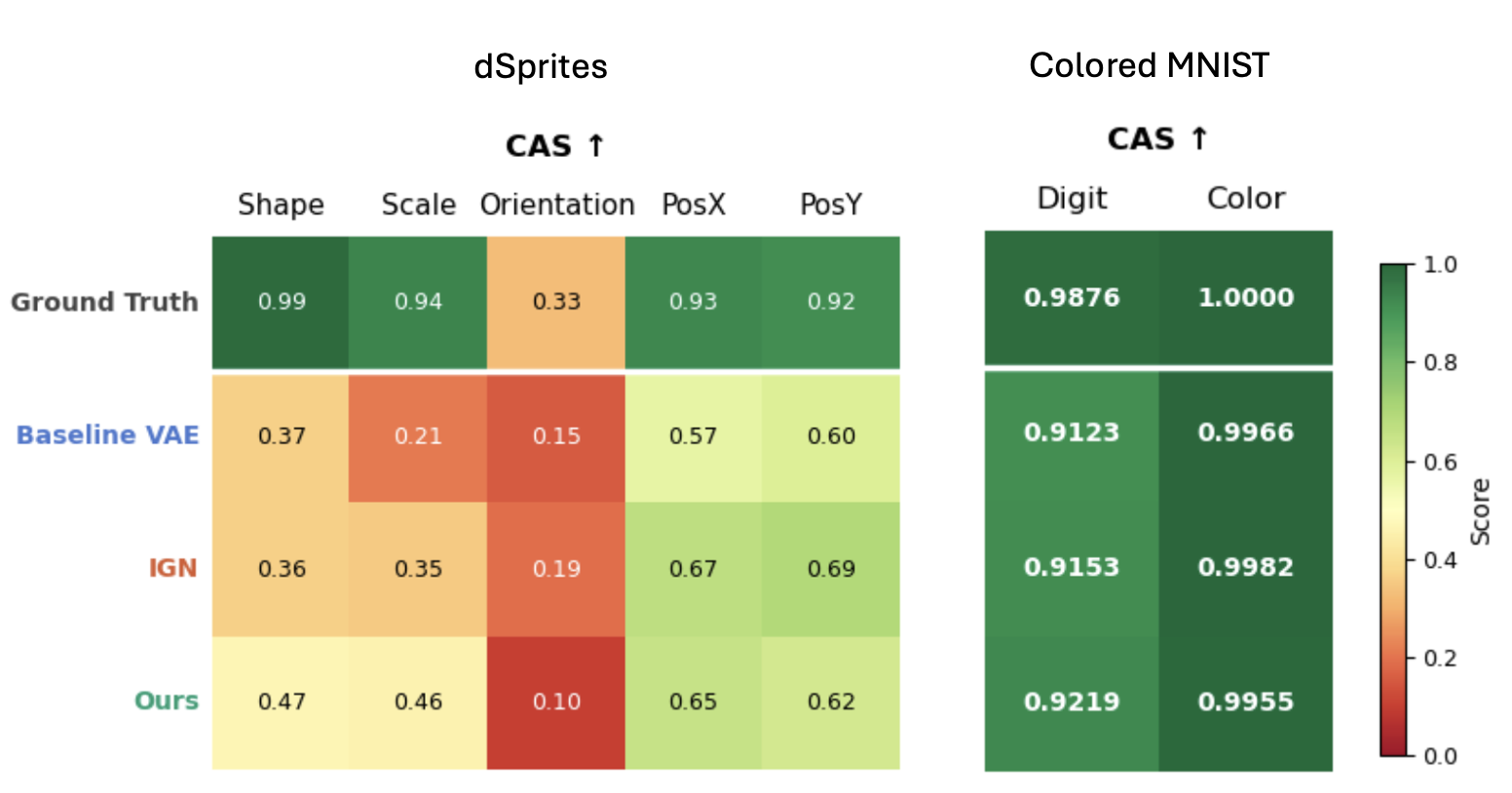}
    \caption{Per-factor CAS breakdown on dSprites and Colored MNIST. Ground Truth
    serves as an upper bound. Our method is competitive across most factors, with
    the largest remaining gap on orientation in dSprites.}
    \label{fig:cas_heatmap}
\end{figure}
\begin{figure}[t]
    \centering
    \subfigure[VAE]{\includegraphics[width=0.32\linewidth]{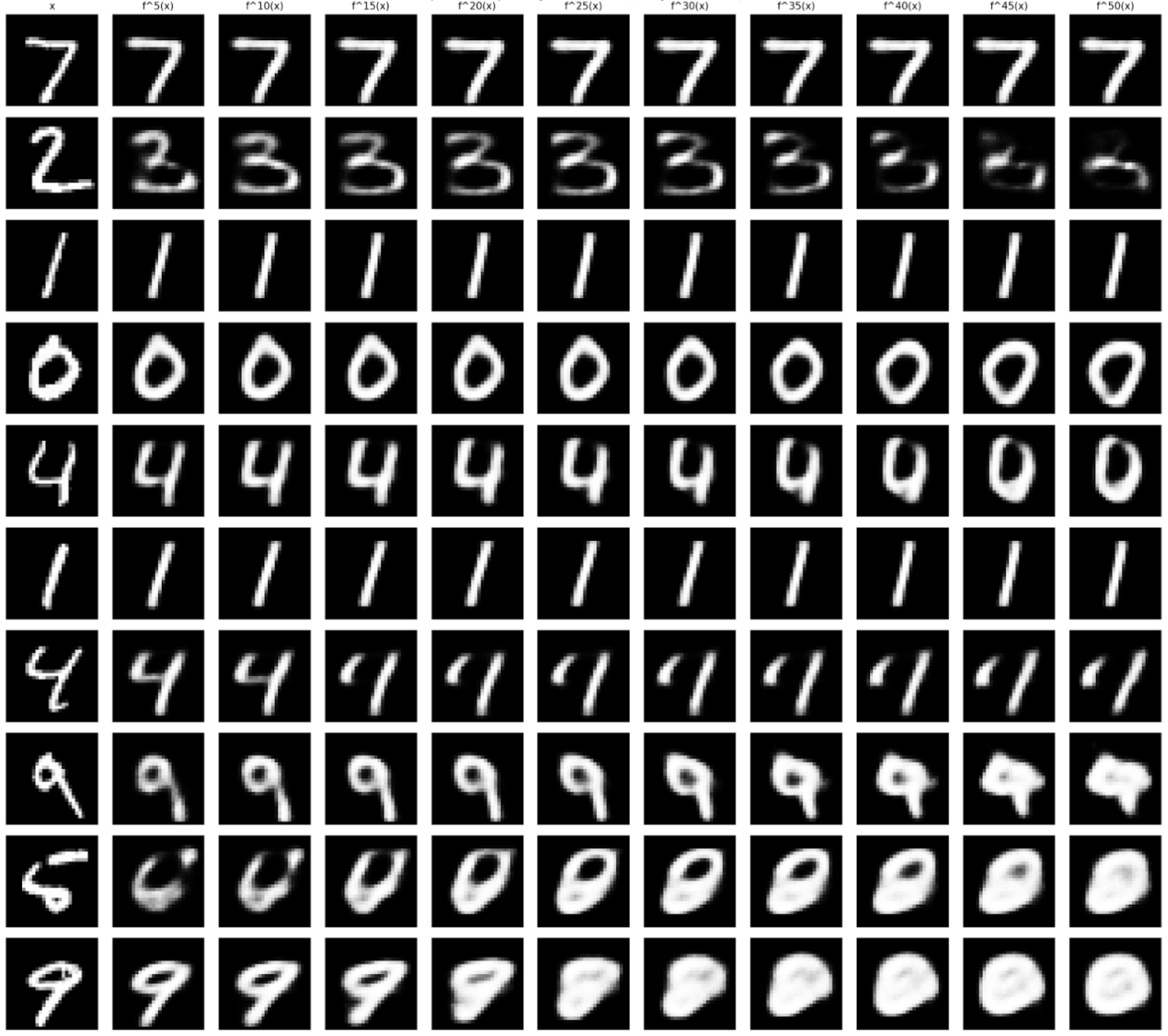}\label{fig:vae}}\hfill
    \subfigure[VAE + IGN]{\includegraphics[width=0.32\linewidth]{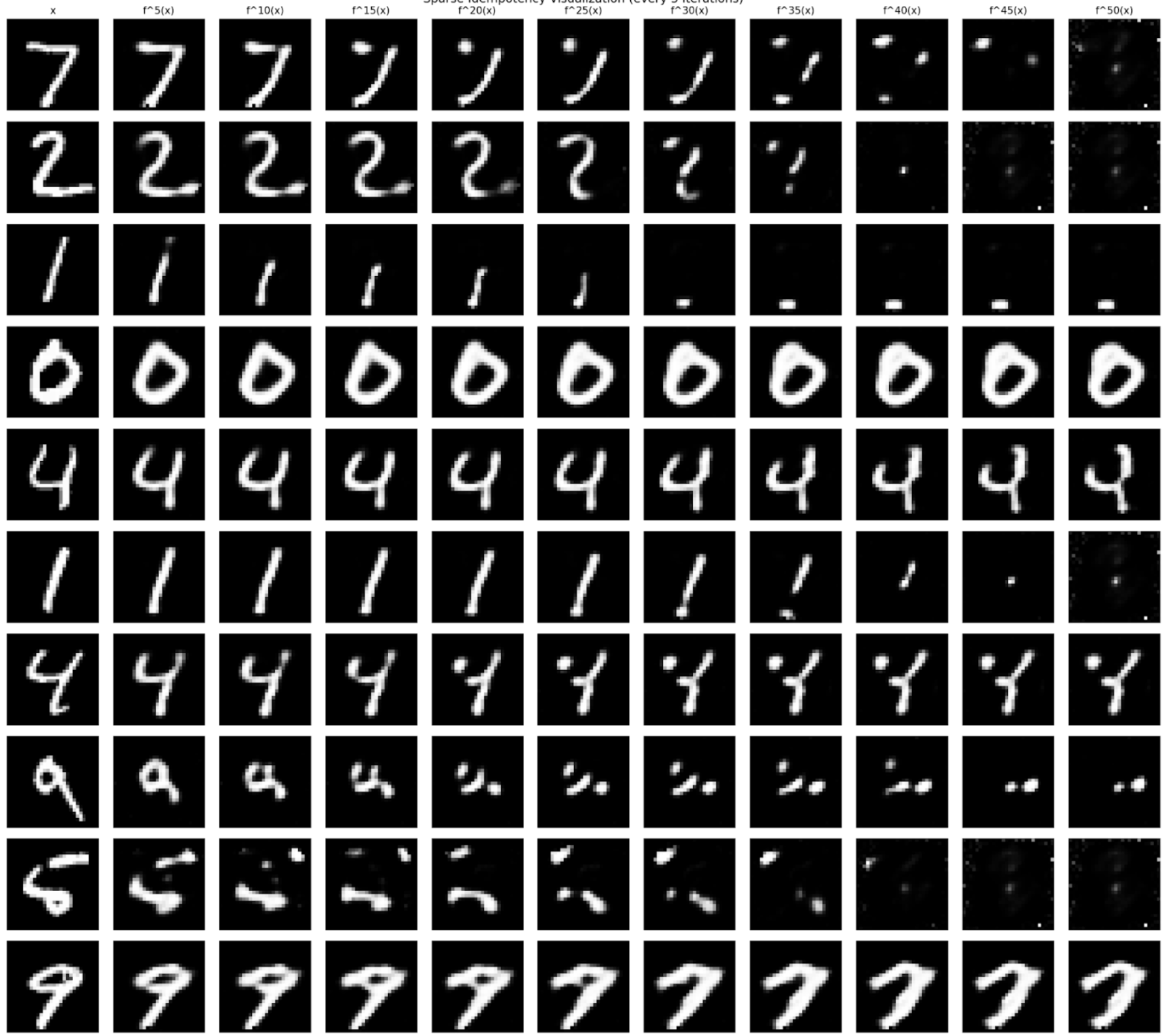}\label{fig:ign}}\hfill
    \subfigure[VAE + Ours]{\includegraphics[width=0.32\linewidth]{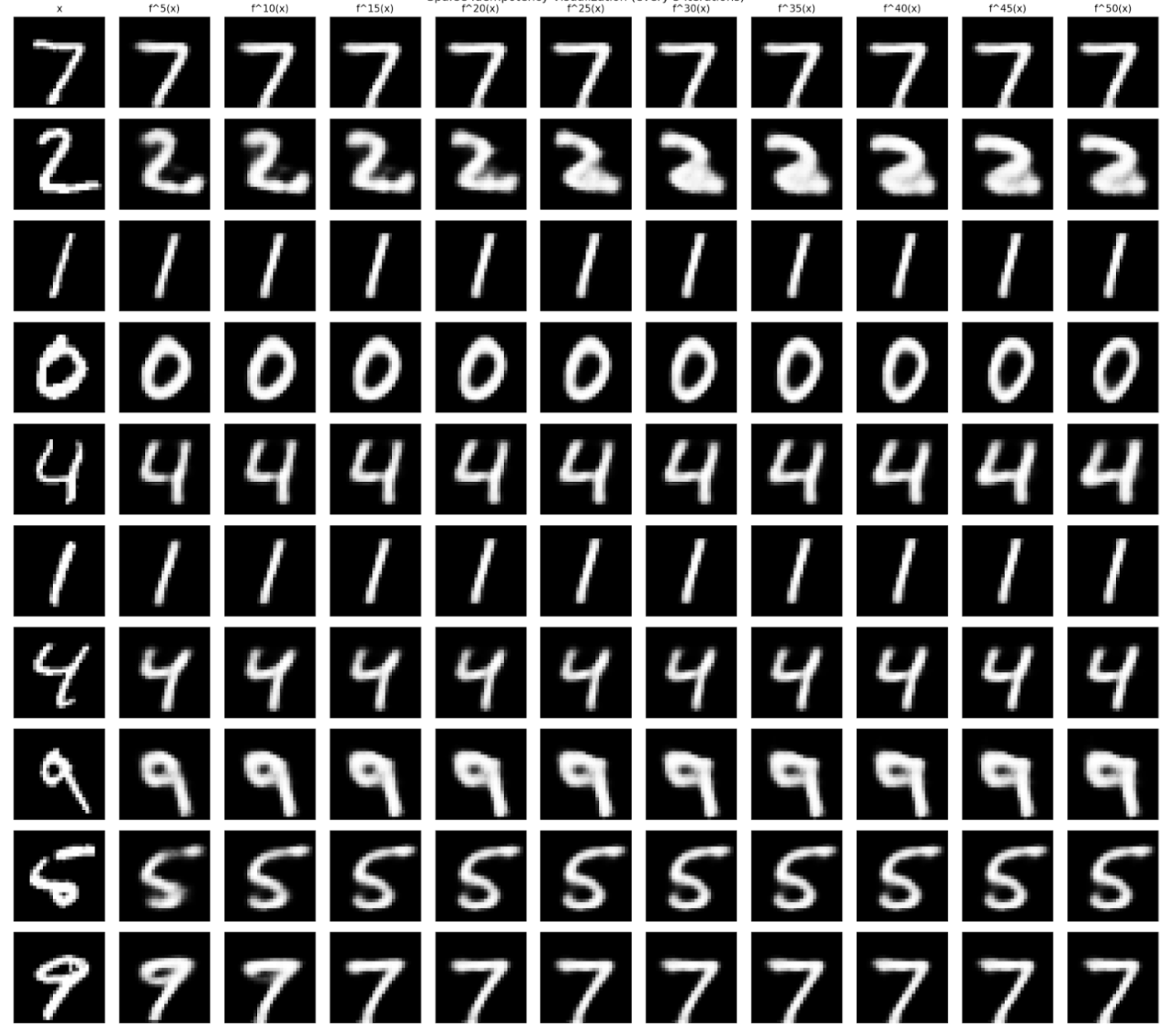}\label{fig:ours}}
    \caption{Qualitative reconstruction comparison on MNIST. Our method produces
    more consistent and stable reconstructions across iterations.}
    \label{fig:qualitative}
\end{figure}

\begin{figure*}[t]
    \centering
    \begin{minipage}[b]{0.48\textwidth}
        \centering
        \includegraphics[width=\linewidth]{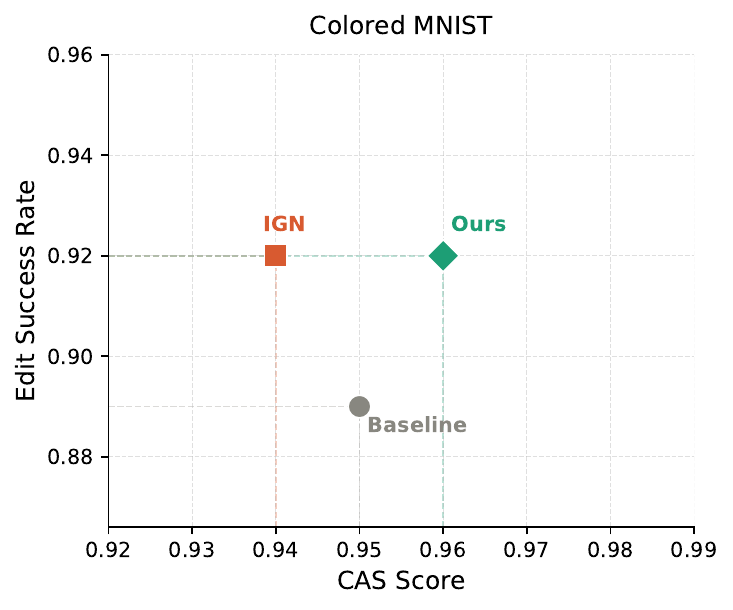}
    \end{minipage}\hfill
    \begin{minipage}[b]{0.48\textwidth}
        \centering
        \includegraphics[width=\linewidth]{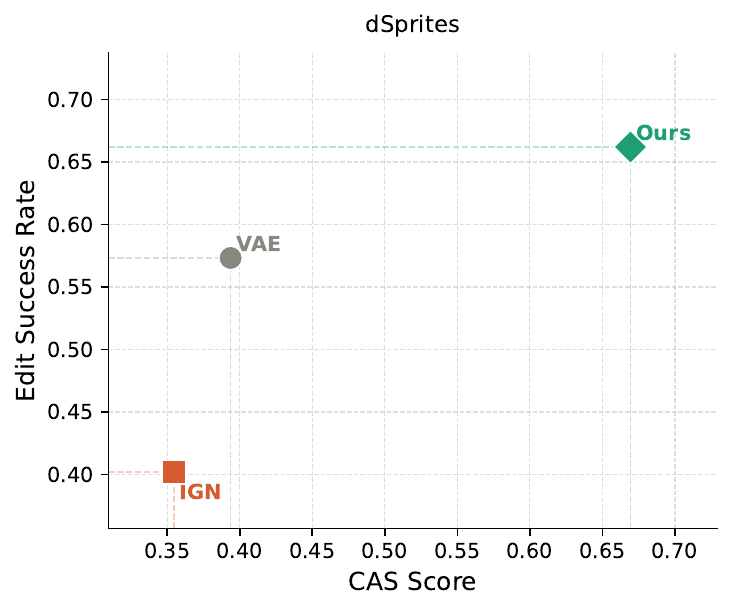}
    \end{minipage}
    \vspace{-0.5em}
    \caption{Attribute editing evaluation on Colored MNIST and dSprites.
    Each point represents a VAE-based variant ( VAE, VAE~+~IGN, and VAE~+~Ours ) 
    scored by Classification Accuracy Score (CAS~\cite{ravuri2019classification})
    and Edit Success Rate (ESR); higher is better for both.}
    
    \label{fig:editing_eval}
    \vspace{-1.0em}
\end{figure*}

\begin{figure*}[t]
    \centering
    \includegraphics[width=\linewidth]{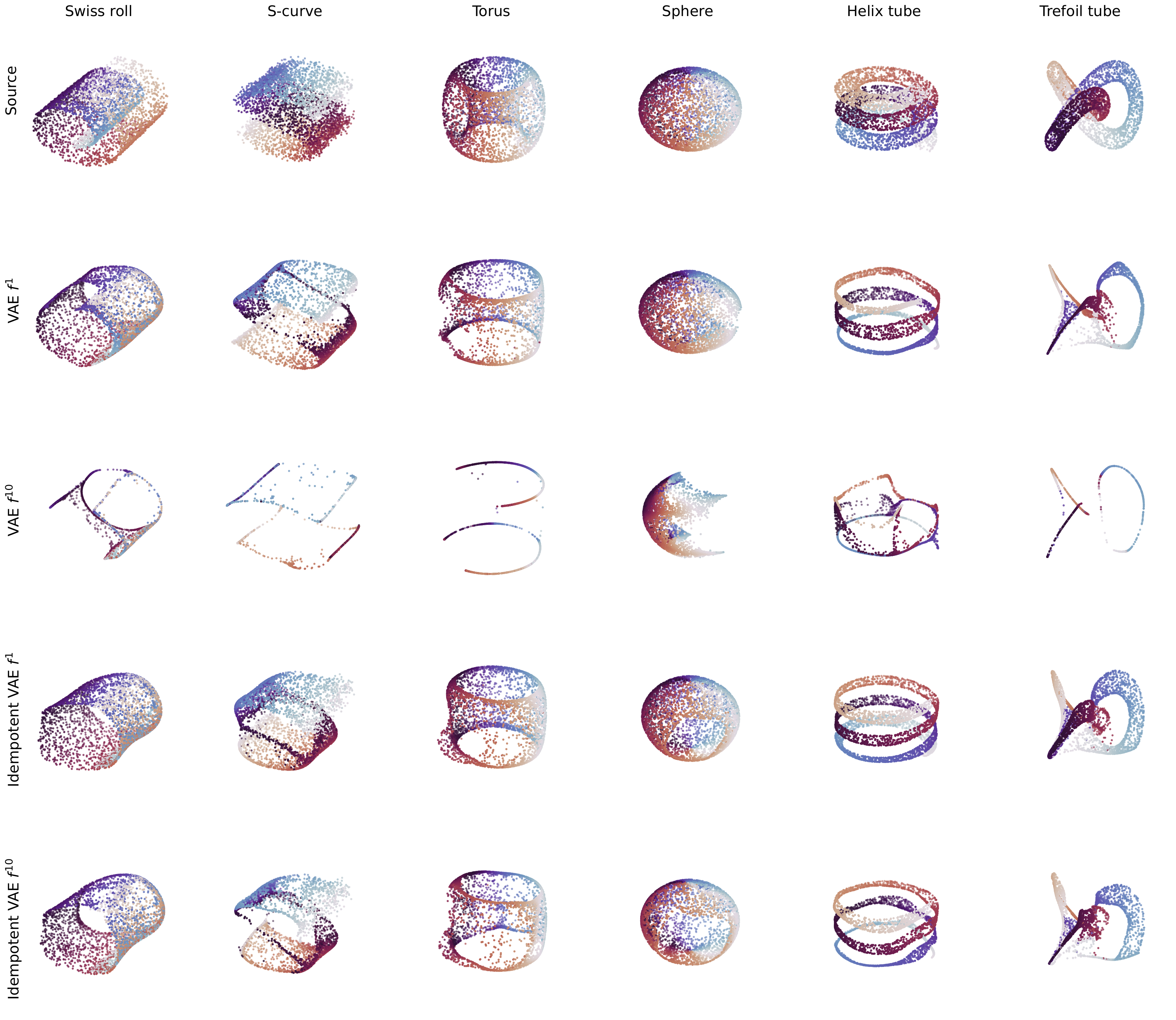}
    \caption{A comparison of the baseline VAE and the idempotent model after
repeated reconstruction rounds, each applying $f = D \circ E$. Columns are the
six surfaces; the top row is the source data, and the remaining rows show the
output after one and ten rounds for each model. Colour encodes position on the
source manifold and is carried through the iterations, so mixing indicates
that points have moved relative to one another. After a single round both
models reproduce the surface, but by the tenth the baseline has collapsed onto
a thin subset while the idempotent model preserves the shape.}
    \label{fig:manifold_grid}
\end{figure*}

\begin{figure*}[t]
    \centering
    \includegraphics[width=\linewidth]{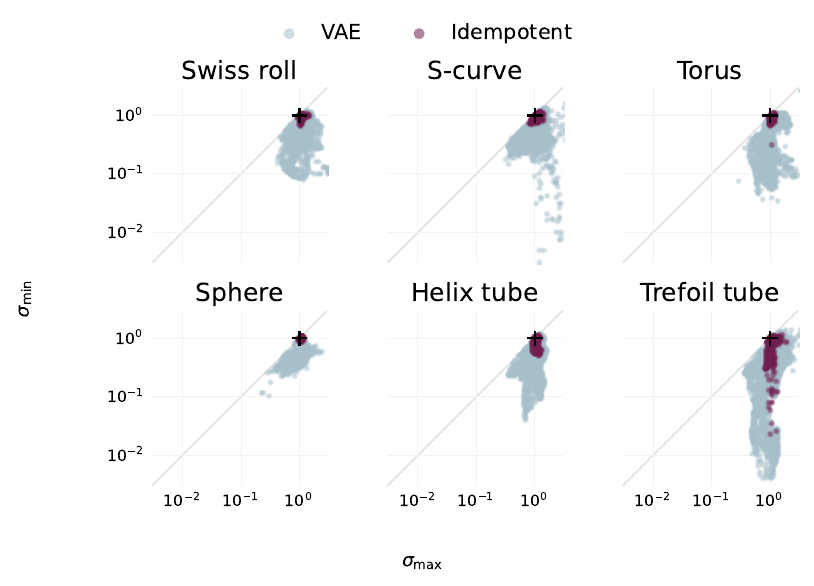}
\caption{A comparison of the round-trip Jacobian $J_E J_D$ between the
baseline VAE and the idempotent model, evaluated at 200 points on each
surface. Each dot is one point, plotted at its largest and smallest singular
value; the cross marks $(1,1)$, where a perturbation returns unchanged in both
directions. Dots off the diagonal indicate that one direction is preserved
while the orthogonal one is contracted. The baseline spreads away from the
cross toward small $\sigma_{\min}$ rather than along the diagonal, showing
that the decoder moves the surface in directions the encoder does not respond
to, while the idempotent model concentrates near identity.}
    \label{fig:jacobian}
\end{figure*}

\begin{figure*}[t]
    \centering
    \includegraphics[width=\linewidth]{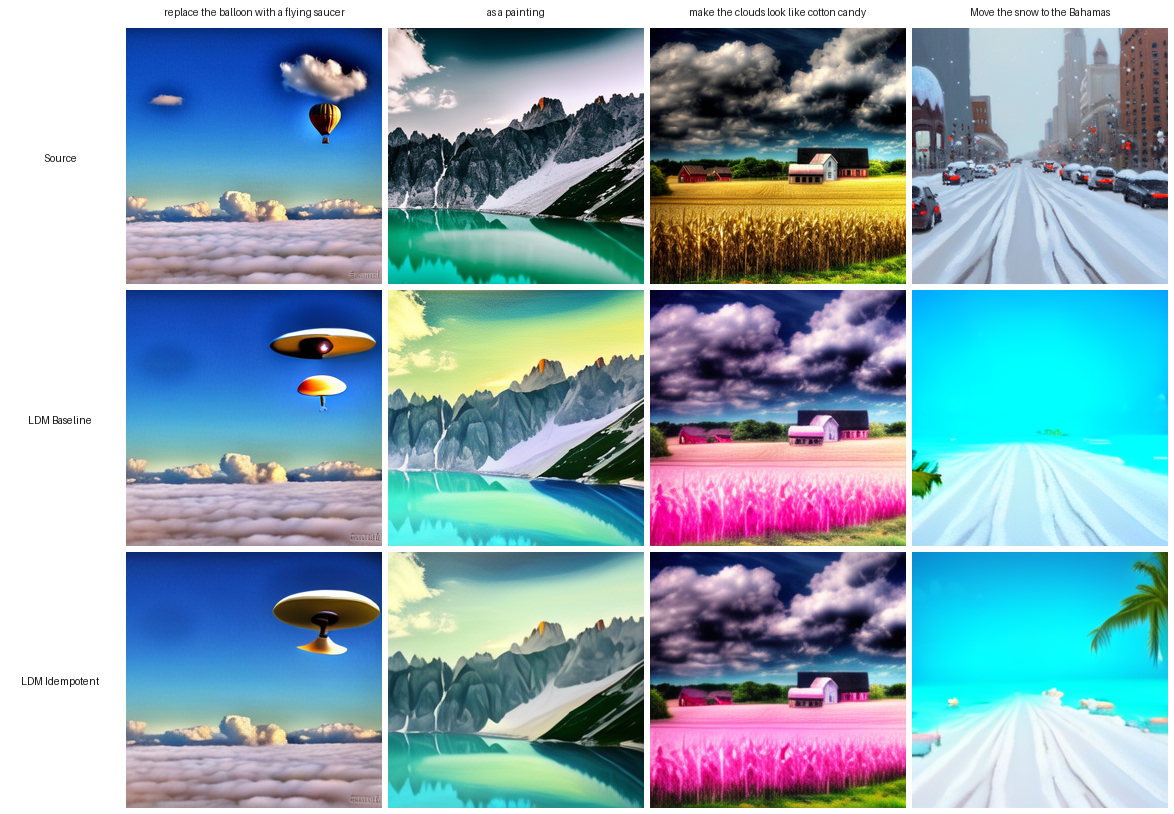}
    \caption{Instruction-guided editing using a latent diffusion
    model~\cite{brooks2023instructpix2pix} trained on idempotent (bottom) versus
    standard VAE representations (middle). The idempotent representation better
    preserves fine-grained details including object positions, colors, and scene
    structure.}
    \label{fig:diff_edit}
\end{figure*}

\end{document}